\documentclass[11pt]{article}
\usepackage[style=alphabetic,natbib=true]{biblatex}
\usepackage{microtype}
\usepackage{geometry}
\usepackage{amsmath,amsfonts,amssymb,amsthm}
\usepackage{graphicx}
\usepackage{tcolorbox}
\usepackage{enumerate}
\usepackage{bbm}
\usepackage{tikz}
\usepackage{verbatim}
\usepackage{breqn}
\usepackage{algorithm}
\usepackage{algpseudocode}
\usepackage{hyperref,color}
\usepackage{cases}
\usepackage{turnstile}
\hypersetup{colorlinks=true,citecolor=blue, linkcolor=blue, urlcolor=blue}
\usepackage{thm-restate}

\makeatletter
\newcommand{\newreptheorem}[2]{\newtheorem*{rep@#1}{\rep@title}\newenvironment{rep#1}[1]{\def\rep@title{#2 \ref*{##1}}\begin{rep@#1}}{\end{rep@#1}}}
\makeatother

\newtheorem{theorem}{Theorem}

\newtheorem{defn}[theorem]{Definition}
\newtheorem{lemma}[theorem]{Lemma}
\newtheorem{coro}[theorem]{Corollary}

\newtheorem{claim}[theorem]{Claim}

\newtheorem{iconj}[theorem]{Informal Conjecture}

\newcommand{\poly}{\mathrm{poly}}

\newcommand{\E}{\mathbb{E}}
\newcommand{\Var}{\mathrm{Var}}

\newcommand{\norm}[1]{\left\lVert #1 \right\rVert}

\def\calD{\mathcal{D}}

\def\eps{\epsilon}

\newcommand{\R}{\mathbb{R}}

\newcommand{\iprod}[1]{\langle#1\rangle}

\newcommand{\He}{\mathrm{He}}

\newcommand{\Qrot}{Q_{\mathrm{rot}}}
\newcommand{\Pv}{\mathcal{P}_v}

\newcommand{\PLA}{\textsc{PLANTED}}
\newcommand{\NULL}{\textsc{NULL}}

 \newcommand{\NO}{\NULL}
 \newcommand{\YES}{\PLA}

\newtheorem{itheorem}[theorem]{Informal Theorem}

\newcommand{\Psymb}{\mathbb{P}}

\newcommand{\conv}{\mathrm{conv}}

  \DeclareMathOperator*{\ProbOp}{\Psymb}

\renewcommand{\Pr}{\ProbOp}

\newcommand{\LDA}{\textrm{LDA}}

\newif\ifnotes\notestrue

\ifnotes
\usepackage{color}
\definecolor{mygrey}{gray}{0.50}
\newcommand{\notename}[2]{{\textcolor{blue}{\footnotesize{\bf (#1:} {#2}{\bf ) }}}}

\newcommand{\anote}[1]{{\notename{Aravindan}{#1}}}

\else

\newcommand{\notename}[2]{{}}

\newcommand{\He}[1]{}
\newcommand{\anote}[1]{}

\fi

\title{Low-Degree Method Fails to Predict Robust Subspace Recovery }

\author{}
\author{He Jia \\ Northwestern University \\ he.jia@northwestern.edu \and Aravindan Vijayaraghavan \\ Northwestern University \\ aravindv@northwestern.edu}
\date{}
\begin{document}

\maketitle

\begin{abstract}
    The low-degree polynomial framework has been highly successful in predicting computational versus statistical gaps for high-dimensional problems in average-case analysis and machine learning. This success has led to the low-degree conjecture, which posits that this method captures the power and limitations of efficient algorithms for a wide class of high-dimensional statistical problems. 
    
    We identify a natural and basic hypothesis testing problem in $\R^n$ which is polynomial time solvable, but for which %
    the low-degree polynomial method fails to predict its computational tractability even up to degree $k=n^{\Omega(1)}$. Moreover, the low-degree moments match exactly up to degree $k=O(\sqrt{\log n/\log\log n})$.  
 
    Our problem is a special case of the well-studied robust subspace recovery problem. The lower bounds suggest that there is no polynomial time algorithm for this problem. In contrast, we give a simple and robust polynomial time algorithm that solves the problem (and noisy variants of it), leveraging anti-concentration properties of the distribution. Our results suggest that the low-degree method and low-degree moments fail to capture algorithms based on anti-concentration, challenging their universality as a predictor of computational barriers.
\end{abstract}

\section{Introduction}
Several problems in high-dimensional statistics and machine learning exhibit striking {\em statistical vs computational gaps} --- these are gaps in the settings where it is possible to recover hidden structure or signal from polynomial samples, yet no known polynomial time algorithms succeed in doing so. Such statistical-computational gaps have been observed in a wide range of problems including planted clique, sparse PCA, community detection, low-rank matrix and tensor decomposition and estimation problems, mixtures of Gaussians, robust estimation and many others \citep{barak2019nearly,berthet2013computational,hopkins2017bayesian,diakonikolas2017statistical, hopkins2018, kunisky2019notes, wein2025computational}. 
A central goal in the theory of average-case analysis is to predict when such statistical-computational gaps arise, and to understand the limits of efficient algorithms for such statistical inference problems. Unlike the worst-case setting, where reductions and hardness conjectures (such as $P \ne NP$) organize our understanding, there has been less success in establishing strong computational hardness assuming a few central assumptions in the average-case world, barring some notable exceptions~\citep{regev2009lattices,bruna2021continuous,brennan2019optimal, brennan2018reducibility}. The predominant approach has been to establish lower-bounds against specific families of algorithms, aiming to characterize the limits of broad algorithmic paradigms through lower bounds against statistical query algorithms, low-degree polynomials, convex relaxation hierarchies, local algorithms, and stable algorithms~\citep{kearns98, feldman2017statistical, diakonikolas2023sq,wein2025computational,hopkins2017power, gamarnik2021}.    

The low-degree polynomial method has emerged as among the most popular approaches to understand statistical vs computational gaps. Consider a hypothesis testing problem, where the goal is to distinguish between a {\em planted} model $P$ and a {\em null or reference} model $Q$ . In this paper we consider a problem where we are given $m$ i.i.d. samples $X_1, \dots, X_m \in \R^n$ where either $X_1, \dots, X_m \sim P$ or $X_1, \dots, X_m \sim Q$. The Low-Degree Advantage of degree $k$ is given by 
\begin{equation}\label{eq:LDA}
\LDA^{(m)}_{\le k}(P,Q) = \max_{f: \text{degree-}k \text{ polynomial}}\frac{\left|\E_{P^{\otimes m}}[f(X_1, \dots, X_m)] - \E_{Q^{\otimes m}}[f(X_1, \dots, X_m)]\right|}{\sqrt{\Var_{Q^{\otimes m}}[f(X_1, \dots, X_m)]}}
\end{equation}
Intuitively, this quantity measures how well degree-$k$ polynomials %
of the samples can distinguish between the two distributions. The low-degree advantage also approximates the likelihood ratio between $P$ and $Q$ by restricting to low-degree statistics. 

When the low-degree advantage is small e.g., $\LDA^{(m)}_{\le k}(P,Q) = O(1)$ or even better $o(1)$, it is considered an evidence of average-case hardness~\citep{kunisky2019notes,wein2025computational, Kunisky2025}. In this regime, every degree-$k$ polynomial fails to distinguish between 
$P$ and $Q$, suggesting that no algorithm with comparable computational power can succeed. 
While the low-degree method does not capture Gaussian elimination and certain lattice-based algorithms~\citep{diakonikolas2022non,zadik2022lattice}, these algorithms exploit rich algebraic structure and are brittle to a small amount of random noise. This is unlike low-degree polynomials which are inherently robust to random noise--- a property that is particularly relevant in statistical estimation problems. 

A rich line of work~\citep{barak2019nearly,hopkins2017bayesian,hopkins2017power} led to the low-degree conjecture (Hypothesis 2.1.5 and Conjecture 2.2.4) in the remarkable thesis of Hopkins ~\citep{hopkins2018} (see also Conjecture 1.6 in \citet{DKWB2021}, Conjecture 2.3 in \citet{Moitra2023PreciseER}), which in our setting can be phrased as follows:

\begin{iconj}
For ``natural'' high-dimensional testing problems specified by $m$ i.i.d. samples from $P,Q$, if $\LDA^{(m)}_{\le k}$ remains bounded as $m \to \infty$, then there is no (noise-tolerant) algorithm running in time $\exp(O(k/ \textrm{polylog}(m)))$ that solves the distinguishing problem. %
\end{iconj}

The above conjecture is usually stated in a more general setting where the testing problem is specified by distributions $P_N$ and $Q_N$ with input size $N$; our setting with $m$ i.i.d. samples from a distribution ($P$ or $Q$) over $\R^n$ corresponds to an input size $N=mn$.     
This conjecture is informal since it does not specify a precise class of testing problems for which it is likely to hold. While the original conjecture was stated for a discrete domain (see Conjecture 2.2.4 in \citet{hopkins2018} for a precise version), the low-degree conjecture has been used for a much broader class of high-dimensional statistical problems over both discrete and continuous domains~\citep[see e.g.,][]{DKWB2021, Moitra2023PreciseER, AV2023, Kunisky2025}.  

Conversely it often holds in practice that if $\LDA^{(m)}_{\le k} \to \infty$ as $m,n \to \infty$, then there is typically a noise-tolerant distinguishing algorithm of running time roughly $(nm)^{O(k/ \textrm{polylog}(mn))}$.  
The low-degree advantage seems to capture most existing algorithmic techniques including moment methods, spectral methods, local methods like approximate message passing (AMP), belief propagation, polynomial threshold functions and even certain Sum-of-Squares (SoS) relaxations~\citep{kunisky2019notes, Kunisky2025, wein2025computational, mohantyetal2026}. %
This conjecture has shown remarkable predictive power across a range of canonical problems in high-dimensional statistics including planted clique, sparse PCA, tensor PCA, clustering, spiked Wigner and Wishart models, tensor decomposition, community detection, Gaussian mixture models, certifying restricted isometry property and many others~\citep[see e.g., ][]{hopkins2017power, DKWB2021, Moitra2023PreciseER, buhai2023semirandom, AV2023, diakonikolas2023sq}. See also the excellent surveys by ~\citet{wein2025computational, Kunisky2025}. 

However, a recent exciting result of Buhai, Hsieh, Jain and Kothari~\citep{buhai2025quasi} provides a counterexample to the quasi-polynomial time variant of a concrete low-degree conjecture~\citep{hopkins2018,holmgren2020counterexamples, wein2025computational}, by encoding a problem on list-decoding error-correcting codes into a statistical estimation task. \citet{buhai2025quasi} show that there is no low-degree advantage even for degree $k = n^{\Omega(1)}$, yet there is a quasi-polynomial-time algorithm for the problem. They also provide an example of a certain testing problem (with a single $n \times n$ matrix as input), that can be solved using eigenvalue computation but exhibits vanishing low-degree advantage for $k=\widetilde{O}(n^{1/3})$.    
Despite the counterexample, the low-degree method %
 remains one of the most compelling frameworks for understanding the limits of efficient algorithms in high-dimensional inference. This leads to the following compelling question:

\vspace{5pt}
 \noindent {\bf Question:} {\em Does the Low-Degree Advantage accurately predict the extent of natural algorithmic approaches for high-dimensional statistical problems? Are there algorithmic techniques that are not captured by the LDA method?  
 }

\subsection{Our Result}

Our main contribution is a natural high-dimensional statistical problem where the Low Degree Advantage fails to predict computational tractability. Our problem is in fact a special case of the well-studied {\em robust subspace recovery} problem~\citep{HardtM2013, Lerman_2018, BCPV, bakshi2021list, gao2026ellipsoid}. In the general version of the problem, there is a (unknown) $d$-dimensional subspace $S \subset \R^n$, that we wish to recover. We are given samples $X_1, \dots, X_m \in \R^n$, where an $\alpha$ fraction of the samples are drawn from a distribution that is entirely supported on $S$, while the rest of the samples are drawn from a distribution fully supported on $\R^n$. The robust subspace recovery problem has an elegant polynomial time algorithm for sufficiently small $d$ that leverages sufficient anti-concentration properties of the null distribution~\citep{HardtM2013, BCPV}. In contrast we identify natural distributions $P$ (planted) and $Q$ (null), which fools the 
low-degree method up to degree $k=n^{\Omega(1)}$, and other moment-based approaches. 
We also give an elementary polynomial time algorithm in this setting that achieves high noise tolerance.

We start by describing the hypothesis testing problem. We are given i.i.d. samples from a distribution $\calD$. Our goal is to distinguish the following two hypotheses:

\begin{itemize}
\item \NO: $\calD$ is the rotationally invariant distribution $\Qrot$, given by a {\em scale mixture of spherical Gaussians}. Formally, a sample $X \in \R^n$ is drawn by first $\lambda\sim N(0,1)$ and then $X \sim N(0, \lambda^2 I)$.
\item \YES: %
$\calD$ is a distribution that places at least an $\alpha=1/\poly(n)$ probability mass on a subspace $S \subset \R^n$ of dimension at most $d=O(1)$. %
\end{itemize}

\begin{figure}[h]
    \centering
    \includegraphics[width=0.7\textwidth]{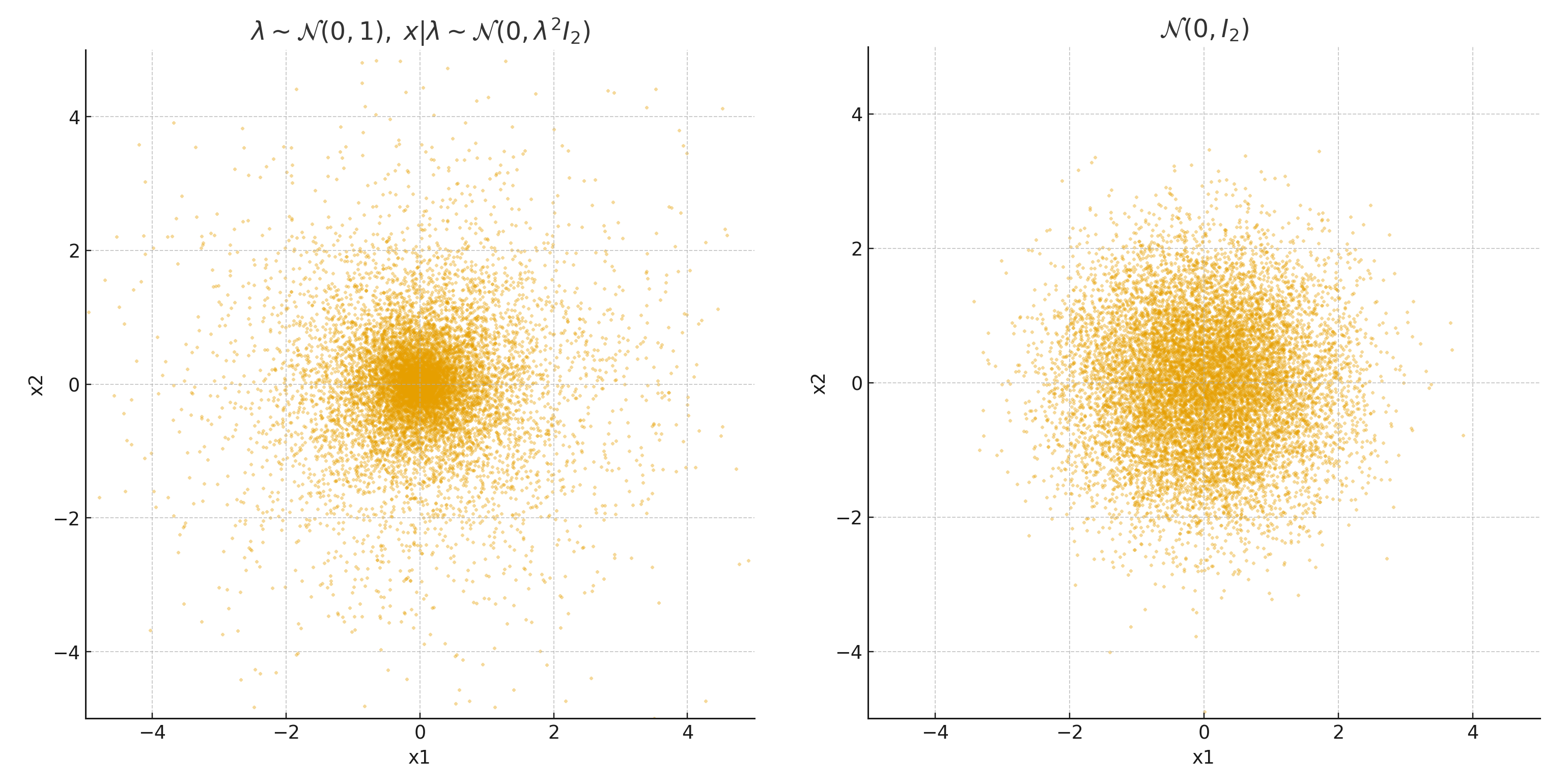}
    \caption{Comparison between a 2D scale mixture of Gaussians and a 2D standard Gaussian. The radial distribution (norm) under the scale mixture is more anti-concentrated.}
    \label{fig:scale-mixture}
    \vspace{-10pt}
\end{figure}

Our planted distribution has a significant mass on a $d$ dimensional subspace $S$; this setting with $d=O(1)$ corresponds to the easiest setting of the robust subspace recovery problem. In fact, the lower-bound result even works when the subspace dimension $d=0$ i.e., $S=\{0\}$.  The null distribution $\Qrot$ is rotationally invariant: there is no low-dimensional subspace which has non-zero probability mass. Moreover, there is an elementary algorithm that distinguishes between the two distributions: take $O(d/\alpha)$ samples, and check to see if there are $d+1$ points lying on a $d$-dimensional subspace!

Our first result shows that the moments of the two distributions $\NULL$ and $\YES$ exactly match up to degree $k=O(\sqrt{\log n/ \log\log n})$. This also implies a low-degree polynomial lower bound up to the same degree.

\begin{theorem}[Moment Matching up to degree $k=\widetilde{O}(\sqrt{\log n})$]\label{ithm:main}
    The two distributions over $\R^n$ given by the \NO~ model and the family of planted distributions $\{\Pv: v \in \mathbb{S}^{n-1}\}$ parametrized by unit vectors $v \in \mathbb{S}^{n-1}$ as described above, satisfy the following properties:
    \begin{enumerate}
        \item $\Qrot$ is rotationally invariant, while $\Pv$ is a mixture: with probability $\alpha = 1/\poly(n)$, $X$ is drawn from a distribution supported on $d=1$ dimensional subspace along $v$; with probability $1-\alpha$, $X$ is drawn from a full-dimensional distribution.
        \item The first $k$ moments of $\Pv$ and $\Qrot$ match for $k=O(\sqrt{\log n /\log\log n})$. Consequently, for any degree-$k$ polynomial $f:\R^n\to \R$,
        \[
        \lvert \E_{X\sim \Pv}[f(X)] - \E_{X\sim Q}[f(X)]\rvert = 0.
        \]       
        \item There is an algorithm that runs in polynomial time (in $n$) and distinguishes between $P_v$ and $\Qrot$ with probability at least $1-o(1)$. Moreover this algorithm is robust, both to adversarial perturbations of the points, and to rerandomizing a $p\in [0,1)$ fraction of the samples from $P_v$ to be drawn from $\Qrot$.
    \end{enumerate}
\end{theorem}
See Theorem~\ref{thm:moment-matching} for the formal statement. 
The above lower bound shows that algorithms that rely on statistics of low-degree polynomials on the samples of degree $k=\widetilde{O}(\sqrt{\log n})$\footnote{We use the notation $\widetilde O(f(n))$ to denote $O(f(n)\poly\log(f(n)))$, suppressing polylogarithmic factors of the argument.} cannot distinguish between the two distributions. In fact $\LDA^{(m)}_{\le k} =0$, and it is independent of the number of samples $m$. This rules out many polynomial time algorithms based on moment-based approaches --- for statistical problems in $n$ dimensions, algorithms based on degree $k$ moments typically provide $n^{O(k)}$ guarantees both in terms of sample complexity and running time. However, %
certain special polynomials of degree $\Theta(\log n)$ can be evaluated in polynomial time; this particularly includes algorithms that are based on good approximations to the top eigenvalue of random matrices.  

Our next result shows that the low-degree advantage is small even when the degree $k=n^{\Omega(1)}$. This suggests the absence of any polynomial time (or even quasi-polynomial time) algorithm, thereby making the polynomial time tractability of our problem interesting.
In what follows, $c>0$ can be chosen to be any constant e.g., $c=1/2$. 

\begin{itheorem}[LDA Lower Bound up to degree $k= n^{\Omega(1)}$]\label{ithm:main2}
    Suppose $c>0$ be any constant. There exist universal constants $c_1>0, c_2>0$ such that the following holds for $\alpha = 1/n^{c}$ and $m= n^{c} \log n$: There exist rotational invariant distributions $\Qrot$ and $P$ over $\R^n$ such that 
\begin{enumerate}
    \item $P$ has $\alpha$ probability mass at $X=0$, while $\Qrot$ places zero probability mass on every subspace of dimension at most $n-1$. 
    \item Given $m$ i.i.d. samples from the distribution and any degree $k\le c_1 n^c/\log n$, the low-degree advantage on the $m$ samples $\LDA^{(m)}_{\le k}(P,\Qrot) \le c_2$.
    \item There is an algorithm that runs in polynomial time  (running time $O(mn)$) and distinguishes between $P$ and $\Qrot$ with probability at least $1-o(1)$. Moreover this algorithm is robust both to adversarial perturbations of the points, and to rerandomizing an $\epsilon$ fraction of the samples from $P$ to be drawn from $\Qrot$.   
\end{enumerate} 
\end{itheorem}
The formal statement showing the bounded low-degree advantage is given in Theorem~\ref{thm:LDA}; in fact, Theorem~\ref{thm:LDA} gives a more general setting of parameters $m, \alpha, k$ where the low-degree advantage is bounded.  %
In the above regime of parameters $m \alpha = \Omega( \log n)$: hence w.h.p. there are many points at $X=0$ (or close to it, in the robust version) that can checked easily in polynomial time. The input is of size $N=mn = n^{1+c}$. Hence, this special case of robust subspace recovery represents a natural statistical problem for which the predictions of the low-degree method fail. 

We remark that the original conjecture of \citet{hopkins2018} for discrete problems imposed additional restrictions of symmetry (permutation-invariance) and a product form for the null distribution $Q_N$. In our setting, the null distribution is comprised of $m$ i.i.d. samples from $\Qrot$, which by itself is rotationally invariant: this is a natural symmetry condition for our problem. While $\Qrot$ is not a product distribution, the null distribution for our input has a product form across the $m$ samples (in fact, it is i.i.d.).
Hence the joint distribution becomes closer to a product distribution when $m \gg n$.

\paragraph{Algorithms and Noise Tolerance.} The complementary algorithmic bound for $d=O(1)$ in the above Theorem~\ref{ithm:main2} is achieved by the simple algorithms in Section~\ref{sec:algo}.  The algorithm samples $m=O(1/\alpha)$ points and checks whether there exists $d+1$ samples which are close to being linearly dependent. Note that the planted distribution for our algorithmic result is more general than the specific planted distribution in the lower bound (Theorems~\ref{ithm:main} and \ref{ithm:main2}). See Theorems~\ref{thm:algo} and \ref{thm:algo:additive} for formal statements. 

Moreover, the algorithm is also noise tolerant: it is robust in several models of noise that are natural for the robust subspace recovery problem. First, a $p \in [0,1]$ fraction of the samples drawn from the $\YES$~distribution can be re-randomized to be drawn from the $\NULL~$distribution $\Qrot$, as in many random noise models for other problems. The work of \citep{holmgren2020counterexamples} gave a counterexample based on error-correcting codes with this noise model, and proposed considering robustness to perturbations of each point in the real domain e.g., a small Gaussian perturbation to every entry when the $\NULL~$distribution is a standard spherical Gaussian $N(0,I)$. 

For our $\NULL$~distribution $\Qrot$, there is no obvious choice for a natural average-case noise model as the points have varying length scales. We instead show robustness of the algorithm against two natural notions of adversarial perturbations $\widetilde{X}$ to each point $X$. The first kind are perturbations whose relative error is upper bounded in magnitude by a constant i.e.,  $\|\widetilde{X} - X\|_2 \le \epsilon \|X\|_2$ for $\epsilon = O_d(1)$. The second kind (additive error worst-case perturbations) bounds the magnitude of the $\|\widetilde{X}-X\| \le \eta$ in an absolute sense (and is more meaningful when $S=\{0\}$). The ranges of $\epsilon, \eta$ allowed by the algorithm are within a constant factor of the best possible even for inefficient algorithms in this model.    
Please see Theorem~\ref{thm:algo} (relative error worst-case perturbations) and Theorem~\ref{thm:algo:additive} (additive error worst-case perturbations) in Section~\ref{sec:algo} for formal statements.

\paragraph{Failure of low-degree polynomials in capturing anti-concentration.} Our results together identify a natural candidate setting (robust subspace recovery) where the low-degree polynomial method fails. Perhaps more interestingly, it suggests that the low-degree method does not adequately capture algorithms that are based on {\em anti-concentration properties} of the distribution. This is unlike concentration properties and tail bound which are often captured by high-degree moments of the distribution. The algorithms for robust subspace recovery~\citep{HardtM2013, BCPV} indeed rely on identifying a small but non-trivial $\alpha$ fraction of samples that are located in a very small region --- a well-spread anti-concentrated distribution is unlikely to have any such high density region. In fact, the recent work of \citep{bakshi2024anticoncentration} studies the problem of certifying anti-concentration with a view towards obtaining guarantees for clustering problems for general distributions. They construct low-degree polynomials to certify that there is no small halfspace of small width $\alpha$ that contains $> C \alpha$ probability mass. However, the degree of the certificate grows as $\Omega(1/\alpha^2)$. For robust subspace recovery, we are in the regime where $\alpha = 1/\poly(n)$. This seems to be challenging for such techniques including \citep{bakshi2021list, bakshi2024anticoncentration}.

\paragraph{New candidate instance for separations}

The main contribution of this work is a natural high-dimensional problem (an average-case setting of robust subspace recovery) which may serve as a candidate for proving separation results between different classes of algorithms. We believe that our candidate instance (and variants) could present a challenge to other algorithm classes including the Statistical Query (SQ) model, and the Sum-of-Squares (SoS). Concretely, we conjecture that the average-case robust subspace recovery problem studied in this work (e.g., in Theorem~\ref{ithm:main}, or a generalization with an $\alpha$-mass on a $d$-dimensional subspace, in the computationally tractable range) is also a hard instance for polynomial time algorithms that fit within the SQ framework.    

This conjecture is motivated by two main observations. First, our low-degree lower bound (Theorem~\ref{ithm:main}) shows that the first $k = O(\sqrt{\log n/ \log\log n})$ moments of the single-sample distributions match exactly. This perfect moment-matching provides strong evidence of indistinguishability for any algorithm that relies on low-degree statistics. Connections to low-degree method lower bounds are also known in certain settings~\citep{brennan21a}. %
Second, our simple algorithms and the existing algorithms for robust subspace recovery~\citep{HardtM2013, BCPV} for sub-constant $\alpha$ rely on identifying approximate linear dependence among a relatively small subset of the sampled points. This procedure is not believed to be captured by the SQ model, which is restricted to estimating aggregate statistical properties of the entire population. Moreover such certificates of anti-concentration seem quite challenging even for powerful algorithm classes like sum-of-squares, where current techniques seem to require degree bounds that are $\poly(1/\alpha)$~\citep{bakshi2024anticoncentration}.

\paragraph{Comparison to Non-Gaussian Component Analysis (NGCA).}
Our results differ from typical hardness results in settings like NGCA \citep{JMLR:v7:blanchard06a,diakonikolas2017statistical,diakonikolas2023sq}. The objective in NGCA is typically to find a single hidden non-Gaussian direction in high dimensions. Lower bound constructions based on the NGCA problem and its variants have played a major role in proving evidence of hardness for several average-case problems~\citep{diakonikolas2017statistical, bruna2021continuous}.  For example, recent work \citep{diakonikolas2023sq} shows that a subspace recovery problem is hard for low-degree polynomials and Statistical Query algorithms when the non-Gaussian signal is confined to a subspace of co-dimension $1$ (i.e., an $(n-1)$-dimensional subspace). 
However, in our constructions, we have an $\alpha$ fraction of mass on a $1$-dimensional or $0$-dimensional subspace; this makes an entire subspace of co-dimension $\ge n-1$ to be non-Gaussian. In fact, this different flavor is expected -- while the NGCA is believed to be computationally hard, our construction where the subspace is low-dimensional is algorithmically easy. This makes it significantly more difficult to design a planted distribution that matches all the mixed moments of the reference distribution. We hope that our new instance which feels significantly different from the NGCA instance will be helpful in understanding the comparative power of different algorithmic frameworks. 

\subsection{Technical Overview}

Our main result is a separation, which we establish in two parts. First, we construct a hypothesis testing problem where a family of planted distributions $P$ and a rotationally invariant null distribution $\Qrot$ are indistinguishable by low-degree polynomials. We establish this indistinguishability in two strong forms: exact moment matching up to degree $k=\widetilde O(\sqrt{\log n})$, and a vanishing Low-Degree Advantage for degrees $k= n^{\Omega(1)}$. Both of these results leverage nice structural properties of the distribution $\Qrot$.  
Second, we provide a simple, efficient algorithm that robustly solves the same problem.

\paragraph{The Moment Matching Construction.}
The core of our first lower bound is the moment-matching claim. We first construct our null distribution $Q$ as a scale mixture of Gaussians. This choice is crucial. A standard Gaussian is rotationally invariant, but its length ($\|X\|_2$) is highly concentrated. 
We require our null distribution $Q$ to satisfy two properties:
(1) {\em rotational invariance}, and
(2) {\em anti-concentration of its length (scale)}.

Our chosen $Q$ (a scale mixture with the scale $\lambda \sim \mathcal{N}(0,1)$) satisfies both. Since both $P$ and $Q$ are designed to be rotationally invariant in the $(n-1)$-dimensional subspace irrelevant to the signal, two consequences follow:
1) The null distribution $Q$ provides no information about the hidden subspace.
2) The moment-matching problem simplifies dramatically. We can analyze the moments along any arbitrary vector in the irrelevant subspace, reducing the $n$-dimensional problem to a 2D problem (one signal coordinate, one irrelevant coordinate).

We then focus on matching the 2D moment vectors of $P$ and $Q$. Our strategy is to give a non-constructive proof to show the existence of $P$ by effectively perturbing the moments of $Q$. The main technical challenge is to prove that this new, perturbed set of moments is valid, i.e., a distribution $\mu_2$ defining them actually exists. In other words, we want to show that the moments of $Q$ are sufficiently far from the boundary of the moment polytope. 
To prove existence, we use Carath\'eodory-style geometric bounds. This shows that our perturbed target moment vector is valid if it lies within the convex hull of all possible moment vectors. We prove this by showing it has a positive Tukey depth. This is where the anti-concentration of the scale/length of $Q$ becomes essential, as this property is inherited by the 2D distribution of the scale of $Q$. 

We first establish a strong, two-sided anti-concentration bound, which relies on the Carbery-Wright inequality. Then a crucial step is to translate this two-sided bound into the one-sided bound required to lower-bound the Tukey depth. This is encapsulated in the following lemma. 

\begin{lemma}[Same as Lemma~\ref{lem:anticonc-onesided}]\label{intro:lem:anticonc-onesided}
Suppose $\nu$ is a probability measure over $\R^n$ that satisfies the following anti-concentration statements: for any multivariate polynomial $p:\R^n \to \R$ of degree at most $k$, 
\[
\forall \delta>0, ~\Pr_{x\sim \nu} \Big[\lvert p(x) -\E_\nu[p] \rvert \le \delta \cdot \sqrt{\Var(p(x))} \Big] \le \eta_k(\delta).
\]
Then for any $\delta_1 >0$ such that $\eta_k(\delta_1)\le 1/2$, we have
\[
\Pr_{x \sim \nu}\Big[  p(x) \ge \E_\nu[p(x)]+\delta \cdot \sqrt{\Var(p(x))} \Big] \ge \frac{\delta_1^2}{8} - \eta_k(\delta).
\]
\end{lemma}

This positive Tukey depth guarantees that our perturbed moment vector is in the interior of the convex hull, which proves the existence of our moment-matching instance up to degree $k=O(\sqrt{\log n/\log\log n})$. These properties of the null distribution $\Qrot$ also help in showing bounded low-degree advantage when for polynomial degree $k$. 

\paragraph{Extending Low-Degree Method Hardness to Polynomial Degrees.}
While our moment-matching construction completely fools polynomials up to degree $k=\widetilde O(\sqrt{\log n})$, standard heuristics suggest that low-degree lower bounds are strong evidence of hardness if they rule out polynomial of degree $k=\omega(\log(n))$. To robustly demonstrate that the low-degree framework fails to predict the tractability of this problem, we establish a bound on LDA for even higher degrees, i.e., $k= n^{\Omega(1)}$. We use the same null distribution $\Qrot$ as in our moment-matching result. Its key structural properties, rotational invariance and strong anti-concentration of its length, remain crucial in proving a bound on the low-degree advantage.

We analyze a simplified planted distribution $P=(1-\alpha)Q+\alpha \Delta_{\bf 0}$, which mixes the null distribution with a point mass at the origin $0$. Using the anti-concentration of the length of $Q$, we first prove that the single-sample advantage $\LDA^{(1)}_{\le k}$ is $O(1)$. A key claim that allows us to prove the bound on the single-sample advantage is the following:

\begin{claim}\label{claim:into:var}
Consider any degree-$k$ polynomial $f$ whose constant term is $0$. Then we have
$$\frac{|\E_Q[f(X)]|}{\sqrt{\Var_Q(f(X))}}\le O(\sqrt{k}),$$
where $Q=\Qrot$ is the rotational invariant scale mixture of Gaussians. 
\end{claim}

We combine this claim with a tensorization argument (in Lemma~\ref{lem:1-to-m-sample}) to lift this to the multi-sample setting, showing that the multi-sample advantage remains bounded by a small constant. 

The bound in the above claim (for a single sample) is independent of the dimension $n$.  We remark that this claim is not true for a spherical Gaussian $N(0,I)$ even though it is rotationally invariant; for example the degree-$2$ polynomial $f(X)=X^\top X$ is concentrated around $n \pm O(\sqrt{n})$. An important property that we exploit about $\Qrot$, we can write $X \sim \Qrot$ as a {\em product} of two independent random variables $X= \lambda g$, where $g \in \R^n$ is a random vector (drawn from a rotationally invariant distribution), and $\lambda$ is a random variable capturing the scale; in our case, it is also drawn according to the normal distribution which has good anti-concentration properties. This allows us to establish the above claim (see Proof of Theorem~\ref{thm:LDA}). We remark that the low-degree lower bound construction also works for $n=1$ dimension (where one can view $g \sim_{u} \{-1,+1\}$). The choice of a normal random variable for $\lambda$ is not crucial; one can use other choices of the scale distribution $\lambda$ that satisfy Claim~\ref{claim:into:var} (potentially with a different dependence on $k$) to achieve other desirable properties.

\paragraph{The Simple Robust Algorithms.}
In contrast to the low-degree hardness, we show that this problem is computationally tractable by presenting two simple algorithms that sample subsets of points and check for approximate linear dependence. For $d$-dimensional hidden subspace, the first algorithm handles adversarial relative perturbations by normalizing samples and detecting if the $(d+1)$-th singular value of any subcollection of size $d+1$ drops below a constant threshold. The second algorithm extends the algorithmic guarantees to the setting of worst-case additive perturbations. Both algorithms rely on the anti-concentration of $\Qrot$ to ensure that random tuples are robustly linearly independent (soundness).

\section{Preliminaries}

\subsection{Lifting Moment Matching to the Multi-Sample Setting}
\paragraph{Moment Matching.} We say two distributions $P$ and $Q$ over $\R^n$ have matching moments up to degree $k$ if for any multivariate monomials $M:\R^n\to\R$ of total degree at most $k$, we have 
\[
\E_{X\sim P}[M(X)] = \E_{X\sim Q}[M(X)].
\]
The definition of LDA involves polynomials over multiple i.i.d. samples from a distribution. The following standard lemma shows that the single-sample moment matching property defined above directly implies that the expectation of any low-degree polynomial over the corresponding multi-sample distributions will also match. This allows us to focus our construction for the low-degree lower bound on a single-sample distribution.
\begin{restatable}{lemma}{momentmatchinglifting}
    Let $P$ and $Q$ be two distributions over $\R^n$ whose moments match up to degree $k$. Let $m$ be the number of samples, and let $P^{\otimes m}$ and $Q^{\otimes m}$ be the product distributions for $m$ i.i.d. samples. Then for any polynomial $f: \R^{m\times n} \to \R$ of total degree at most $k$, we have
    \[
    \E_{(X_1,\dots,X_m)\sim P^{\otimes m}}[f(X_1,\dots,X_m)] = \E_{(Y_1,\dots,Y_m)\sim Q^{\otimes m}}[f(Y_1,\dots,Y_m)].
    \]
\end{restatable}
\begin{proof}
By linearity of expectation, it suffices to prove the equality for any monomial $M(X_1,\dots,X_m)$ of total degree at most $k$. We can write any such monomial $M$ as a product of per-sample monomials: 
\[
M(X_1,\dots,X_m) = \prod_{i=1}^m M_i(X_i)
\]
where the total degree is $\sum_{i=1}^m\deg(M_i)\le k$ and hence each $\deg(M_i)\le k$.
By independence of the samples,
\[
\E_{P^{\otimes m}}[M] = \E_{P^{\otimes m}}\left[\prod_{i=1}^m M_{i}(X_i)\right]
= \prod_{i=1}^m \E_{X_i\sim P}[M_{i}(X_i)].
\]
By the moment matching assumption,
\[
\prod_{i=1}^m \E_{X_i\sim P}[M_{i}(X_i)] = \prod_{i=1}^m \E_{Y_i\sim Q}[M_{i}(Y_i)] = \E_{Q^{\otimes m}}[M].
\]
This proves the claim.
\end{proof}
\subsection{Hermite Polynomials}\label{sec:hermite-polynomials}
We use the standard probabilist's Hermite polynomials, denoted by $\He_m(x)$. For our analysis, we will use the normalized Hermite polynomials, defined as:
    \[
    h_m(x) = \frac{\He_m(x)}{\sqrt{m!}}.
    \]
    The most important property of these normalized polynomials is that they form a complete orthonormal basis for the space of square-integrable functions with respect to the standard Gaussian measure $\mathcal{N}(0,1)$, i.e., for any $i,j\ge 0$,
    \[
    \E_{x\sim\mathcal{N}(0,1)}[h_i(x)h_j(x)]=
    \begin{cases}
    1, & i=j,\\
    0, & i\neq j.
    \end{cases}
    \]
    This orthonormality also extends to higher dimensions. For reference, these normalized Hermite polynomials also have the following standard identity, which relates them to the monomial basis:
    \begin{equation}\label{eqn:hermite-to-monomial}
        h_m(x) = \sum_{j=0}^{\lfloor \frac{m}{2}\rfloor} a_{m,j}x^{m-2j}, ~~~a_{m,j} = \frac{(-1)^j\sqrt{m!}}{2^j j!(m-2j)!}.
    \end{equation}

\section{Moment Matching Instance}\label{sec:moment-matching}
In this section, our goal is to build a planted distribution whose first $k$ moments agree with those of a rotationally invariant null distribution 
$Q$, while still hiding a 1-dimensional signal direction. Concretely, we use the scale parameter $\lambda$ to control the overall spherical spread and a scalar 
$z$ to carry the signal in one coordinate; the construction lets us tune a small mass $\alpha_k$ that lives on a 1-dimensional subspace while matching all mixed moments up to degree $k$.
We also prove that the moments vectors of our planted and null distributions match exactly using the tools that are introduced in Section~\ref{sec:anticonc}. 
\begin{defn}[Scale mixture of Gaussians]
\label{def:scale-GMM}
An $n$-dimensional scale mixture of Gaussians is a distribution of $X$ on $\R^n$ where its scale parameter $\lambda$ is governed by a distribution $\mu$:
\[
\lambda\sim \mu;\; X\mid \lambda\sim\mathcal{N}(0,\lambda^2 I_{n}).
\]
\end{defn}

\begin{restatable}{theorem}{momentmatching}
\label{thm:moment-matching}
    Let $Q$ be a scale mixture of Gaussians with Gaussian scales $\mu=\mathcal{N}(0,1)$ as in Definition~\ref{def:scale-GMM}.
    For any positive integer $k$, there exist $\alpha_k\in (0,1)$ satisfying $\alpha_k\ge k^{-O(k^2)}$, and distributions $\mu_1(z),\mu_2(\lambda,z)$ that define a new distribution $P$, which 
    has its first $k$ moments matching those of $Q$. $P$ is defined as follows for $X=(\bar X, X_n)$ over $\R^n$, where $\bar X\in\R^{n-1},X_n\in\R$:
    \begin{itemize}
        \item with probability $\alpha_k$, $\bar X=0, X_n\sim\mu_1$,
        \item with probability $1-\alpha_k$, $(\lambda,z)\sim\mu_2; \bar X \mid \lambda\sim \mathcal{N}(0,\lambda^2 I_{n-1}), X_n=z$.
    \end{itemize}
\end{restatable}
Our proof strategy is to first establish the moment-matching conditions, which, by rotational invariance, simplifies our task to matching the mixed moments of just two coordinates. We will show that the moments of our planted distribution $P$ can be expressed as a convex combination of two components. The main technical challenge is to prove that the planted distribution, defined by distributions $\mu_1$ and $\mu_2$, can be constructed to satisfy the resulting system of moment equations. To do this, we will reframe the problem as finding a target moment vector $\theta$ within the convex hull of the support of a related moment vector $\Phi$. We then leverage the powerful tools related to anti-concentration and the geometry of moment spaces developed in Section~\ref{sec:anticonc} to prove that such a distribution $\mu_2$ must exist.

\begin{proof}[Proof of Theorem~\ref{thm:moment-matching}]
    We first establish the moment-matching conditions. We want to show that for any unit vector $v=(\bar v, v_n), \bar v\in\R^{n-1},v_n\in\R$ and for any $\ell\in[k]$, 
    \[
    \E_{X\sim P}[\langle X, v\rangle^\ell]=\E_{X\sim Q}[\langle X, v\rangle^\ell].
    \]
    Expanding the inner product gives:
    \[
    \E[\langle X, v\rangle^\ell] = \E\left[(\bar X\cdot \bar v+ X_n v_n)^\ell\right] = \sum_{\ell_1,\ell_2\ge0,\ell_1+\ell_2=\ell} \binom{\ell}{\ell_1}\E[(\bar X\cdot \bar v)^{\ell_1}(X_n v_n)^{\ell_2}].
    \]
    By definition, $Q$ is a mixture of zero-centered Gaussians so it is rotationally invariant. The marginal of the first $n-1$ coordinates of $P$ is also 
    rotationally invariant, 
    so it suffices to show the equality for any $\bar v$. Without loss of generality, we may take $\bar v=e_1$.
    Therefore, it is equivalent to show that for any $\ell_1,\ell_2\ge0$ such that $\ell_1+\ell_2\le k$,
    \begin{equation}\label{eqn:moment-matching-term}
    \E_P[x_1^{\ell_1}x_n^{\ell_2}] =\E_Q[x_1^{\ell_1}x_n^{\ell_2}].
    \end{equation}
    Furthermore, both distributions are centrally symmetric. This implies that any odd moment is zero. Thus, we only need to match the moments where both $\ell_1$ and $\ell_2$ are even. The RHS of (\ref{eqn:moment-matching-term}) is 
    \[   \E_Q[x_1^{\ell_1}x_n^{\ell_2}]=\E_{\lambda\sim\mu}\left[\E_{x\sim\mathcal{N}(0,\lambda^2I)}[x_1^{\ell_1}x_n^{\ell_2}]\right]
    =\E_{\lambda\sim\mu}\left[c_{\ell_1}c_{\ell_2}\lambda^{\ell_1+\ell_2}\right]=c_{\ell_1}c_{\ell_2}c_{\ell_1+\ell_2}
    \]
    where $c_{\ell}$ is the $\ell$-th moment of the standard Gaussian $\mathcal{N}(0,1)$. When $\ell_1=0$, we have the LHS of (\ref{eqn:moment-matching-term}) being
    \begin{align*}
    \E_P[x_1^{\ell_1}x_n^{\ell_2}]
    =\alpha_k \E_{z\sim\mu_1}[z^{\ell_2}]+(1-\alpha_k)\E_{z\sim\mu_2}[z^{\ell_2}].
    \end{align*}
    When $\ell_1\neq0$, we have
    \[
    \E_P[x_1^{\ell_1}x_n^{\ell_2}]= (1-\alpha_k)\E_{(\lambda,z)\sim\mu_2}\left[\E_{x_1\sim\mathcal{N}(0,\lambda^2)}[x_1^{\ell_1}z^{\ell_2}]\right]=(1-\alpha_k)c_{\ell_1}\E_{(\lambda,z)\sim\mu_2}[\lambda^{\ell_1}z^{\ell_2}].
    \]
    Let $\mathcal I=\{(\ell_1,\ell_2):\ell_1,\ell_2 
    \text{ even}, 0<\ell_1+\ell_2\le k\}$ and $D=|\mathcal I|$. Let $\nu(\lambda,z)$ be the distribution defined in Section~\ref{sec:anticonc}. Define random variable $\Phi(\lambda,z)\in \R^D$ with coordinates
    \[
    \Phi_{\ell_1,\ell_2}(\lambda,z)=\lambda^{\ell_1} z^{\ell_2}.
    \]
   Then $\E_\nu[\Phi_{\ell_1,\ell_2}(\lambda,z)]=c_{\ell_2}c_{\ell_1+\ell_2}$. Define our target vector $\theta\in \R^D$
    \[
    \theta_{\ell_1,\ell_2}=\E_{(\lambda,z)\sim\mu_2}[\lambda^{\ell_1}z^{\ell_2}]=
    \begin{cases}
    \frac{c_{\ell_2}^2-\alpha_k\E_{z\sim\mu_1}[z^{\ell_2}]}{1-\alpha_k}, &\ell_1 = 0,\\
    \frac{c_{\ell_2}c_{\ell_1+\ell_2}}{1-\alpha_k}, &\ell_1\ge 1.
    \end{cases}
    \]
    Let $\beta\in \R^D$ be arbitrary vector of coefficients with $\norm{\beta}=1$. Then 
    \[
    p_\beta(\lambda,z) := \langle \beta, \Phi\rangle = \sum_{\mathcal I}\beta_{\ell_1,\ell_2}\lambda^{\ell_1} z^{\ell_2}.
    \]
    and $\E_\nu[p_\beta] = \sum_{\mathcal I}\beta_{\ell_1,\ell_2}c_{\ell_2}c_{\ell_1+\ell_2}$. Let $\theta'\in \R^D$ be
    \[
    \theta'_{\ell_1,\ell_2}=(1+\delta)c_{\ell_2}c_{\ell_1+\ell_2}.
    \]
    So we have $\langle \beta,\theta'\rangle=(1+\delta)\E_\nu[p_{\beta}(\lambda,z)]$.
    By Corollary~\ref{coro:tukey-depth}, we know for any $\delta\le 1/k^{Ck^2}$,
    \[
    h_\Phi(\theta') = \inf_{\beta:\norm{\beta}=1}\Pr_\Phi[\langle \beta, \Phi-\theta'\rangle\ge 0]= \inf_{\beta:\norm{\beta}=1}\Pr[p_\beta\ge (1+\delta)\E_\nu[p]]\ge\eta_k(\delta)>0.
    \]
    By Theorem~\ref{thm:lyons}, $\theta'$ is in the convex set of moment vectors $\Phi(\lambda,z)$ of random $(\lambda,z)\sim\nu$. By Lemma~\ref{lem:depth-interior}, pick $r=r(h_\Phi(\theta'))>0$ so that $B(\theta',r)\subseteq \mathrm{conv}(\mathrm{supp}(\Phi))$.
    Set $\frac{1}{1-\alpha_k}=1+\delta$. For the $\ell_1\ge 1$ coordinates, $\theta'_{\ell_1,\ell_2}=\theta_{\ell_1,\ell_2}$. For $\ell_1=0$, $\theta_{0,\ell_2} = \theta'_{0,\ell_2}-(1+\delta)\alpha_k\E_{z\sim\mu_1}[z^{\ell_2}]$. Choose $\mu_1$ so that $\E_{z\sim\mu_1}[z^{\ell_2}]$ moves the $\ell_1=0$ block by at most $r$ and hence $\theta \subseteq B(\theta',r)\subseteq \mathrm{conv}(\mathrm{supp}(\Phi))$. Therefore, there exists a valid distribution $\mu_2$ on $(\lambda,z)$ with $\E_{\mu_2}[\Phi]=\theta$, i.e., all required mixed moments match as in (\ref{eqn:moment-matching-term}).
\end{proof}

\subsection{Anticoncentration and Tukey Depth}\label{sec:anticonc}
As outlined in the proof of Theorem~\ref{thm:moment-matching}, our moment-matching construction hinges on proving the existence of a specific distribution $\mu_2$. To do this, we must show that a particular target vector of moments $\theta$ lies within the convex hull of the support of a random moment vector $\Phi(\lambda,z)$. This section develops the necessary anti-concentration tools to prove this claim.

We begin by defining a two-dimensional distribution $\nu$ over $(\lambda,z)$, which will serve as the basis for the random moment vector $\Phi(\lambda,z)$ used in Section~\ref{sec:moment-matching}.
Consider the joint distribution $\nu$ over $\{(\lambda,z): \lambda, z \in \R\}$ defined as follows:
\begin{enumerate}
    \item $\lambda \in \R$ is drawn from a standard Gaussian $\mathcal{N}(0,1)$
    \item Conditioned on $\lambda$, $z \in \R$ is drawn from $\mathcal{N}(0,\lambda^2)$. 
\end{enumerate}
In other words, the joint distribution is given by
\begin{align}\label{eq:mu:ref}
\nu(\lambda,z) = \frac{1}{\sqrt{2\pi}} e^{-\lambda^2/2} \cdot \frac{1}{|\lambda| \sqrt{2\pi}} e^{-z^2/(2\lambda^2)}. 
\end{align}

The reference or null distribution over $\R^n$ will be constructed based on the above two-dimensional measure $\nu$. The parameter $\lambda$ will determine the scale of the distribution, and $z$ will determine the component along a special one-dimensional projection.

\subsection{Anticoncentration Bounds}

The following anti-concentration property of the distribution $\nu$ will be important for the lower bound construction. It ensures that no single value (or narrow band) has too much probability mass under any low-degree polynomial of our reference distribution. This property will later be used in Section~\ref{sec:tukey-depth} to establish a positive Tukey depth for our moment vector, which is the key to proving its mean is in the interior of the convex hull of all valid moment vectors. 

\begin{restatable}{lemma}{anticonc}
\label{lem:anticonc}
Suppose $k \in \mathbb{N}$. Consider the degree $k$ polynomial 
\[
p(\lambda,z)= \sum_{\ell_1,\ell_2\text{ even}, 0<\ell_1+\ell_2\le k} \beta_{\ell_1, \ell_2} \lambda^{\ell_1} z^{\ell_2}
\]
where $\beta$ is the vector of coefficients $\beta=\big(\beta_{\ell_1, \ell_2}: \ell_1, \ell_2 \in \{0,1,\dots,k\} \text{ s.t. } \ell_1 + \ell_2 \le k \big)$. Then there exists a universal constant $C$ such that  
\begin{equation}\label{eq:muref:anticonc}
\forall \delta>0, ~\sup_{t \in \R} \Pr_{(\lambda,z)\sim \nu} \Big[\big|p_\beta(\lambda,z) -t \big| \le \delta \cdot |\E[p_\beta]| \Big] \le C \cdot k\cdot \Big(\sqrt{k}\delta\Big)^{\frac{1}{2k}}.
\end{equation}

\end{restatable}

The proof of Lemma~\ref{lem:anticonc} is based on the following Carbery-Wright anti-concentration inequality for Gaussian distributions.
\begin{theorem}[Carbery-Wright, \cite{carbery2001}]
\label{thm:carbery-wright}
Let $p:\R^n \to \R$ be a real-valued polynomial of total degree $k$, and let $\nu$ be a log-concave probability measure on $\R^n$. There exists a universal constant $C$ such that for any $\delta > 0$,
 \[
 \sup_{t \in \R} \Pr_{x \sim \nu} \left[ |p(x) - t| \le \delta \sqrt{ \mathrm{Var}(p(x))} \right] \le C \cdot k \cdot \delta^{1/k}.    
\]
\end{theorem}
To apply Carbery-Wright inequality, we first prove the following lemma that compares the mean and the variance of low-degree polynomials under a Gaussian. The univariate case where $q$ is only a polynomial of $\lambda$ is a special case and is also used in bounding the low-degree advantage in Theorem~\ref{thm:LDA}. 
\begin{restatable}{lemma}{meanvarratio}
    \label{lem:mean-var-ratio}
    Let $q$ be a polynomial in $(\lambda,g)$ of degree at most $k$ with zero constant term. Under $(\lambda,g)\sim N(0,I_2)$,
    \[
    |\E[q(\lambda,g)|\le C\sqrt{k\Var(q(\lambda,g))}.
    \]
\end{restatable}
\begin{proof}
    Let 
    \[
    q(\lambda,g) = \sum_{\ell_1,\ell_2: 0<\ell_1+\ell_2\le k} \beta_{\ell_1, \ell_2} \lambda^{\ell_1} g^{\ell_2}.
    \]
    We can also write $q(\lambda,g)$ using the normalized Hermite basis defined in Section~\ref{sec:hermite-polynomials}. Suppose $\gamma_{i_1,i_2}$ are the coefficients corresponding to $h_{i_1}(\lambda)h_{i_2}(g)$, i.e.,
    \[
    q(\lambda,g) = \sum_{0\le i_1,i_2\le k} \gamma_{i_1,i_2}h_{i_1}(\lambda)h_{i_2}(g).
    \]
     Since $h_{i_1}(\lambda)h_{i_2}(g)$ are orthonormal over the standard Gaussian distribution, 
    \begin{align*}
        \E_{N(0,I_2)}[q(\lambda,g)] &= \sum_{0\le i_1,i_2\le k} \gamma_{i_1,i_2}\E_{N(0,I_2)}[h_{i_1}(\lambda)h_{i_2}(g)]= \gamma_{0,0},\\
    \mathrm{Var}_{N(0,I_2)}(q(\lambda,g)) &= \sum_{0\le i_1,i_2\le k} \gamma_{i_1,i_2}\E_{N(0,I_2)}[h_{i_1}(\lambda)^2h_{i_2}(g)^2]-\gamma_{0,0}^2 =
    \sum_{i_1,i_2\in \mathcal I'} \gamma_{i_1,i_2}^2
    \end{align*}
    where $\mathcal I'= \{(i_1,i_2)\neq 0: i_1,i_2\text{ even}, 0\le i_2\le i_1\le k\}$.
    By the formula (\ref{eqn:hermite-to-monomial}), we have
    \[
    \beta_{0,0} = \sum_{0\le i_1,i_2\le k,i_1,i_2\text{ even}}a_{i_1,\frac{i_1}{2}}a_{i_2,\frac{i_2}{2}}\gamma_{i_1,i_2}, ~~~
    a_{i,\frac{i}{2}} = \left(-\frac{1}{2}\right)^{\frac{i}{2}}\frac{\sqrt{i!}}{\left(\frac{i}{2}\right)!}.
    \]
    Thus,
    \[
    \beta_{0,0} = \gamma_{0,0}+\sum_{i_1,i_2\in I^*}a_{i_1,\frac{i_1}{2}}a_{i_2,\frac{i_2}{2}}\gamma_{i_1,i_2}.
    \]
    Given the constant term $\beta_{0,0}=0$, we have the identity
    \[
    \gamma_{0,0}=-\sum_{i_1,i_2\in I^*}a_{i_1,\frac{i_1}{2}}a_{i_2,\frac{i_2}{2}}\gamma_{i_1,i_2}.
    \]
    Then by Cauchy-Schwarz, we have
    \[
    \frac{\gamma_{0,0}^2}{\sum_{i_1,i_2\in I^*}\gamma_{i_1,i_2}^2}
    \le \sum_{i_1,i_2\in I^*}a_{i_1,\frac{i_1}{2}}^2a_{i_2,\frac{i_2}{2}}^2.
    \]
    By Stirling's approximation, 
    \[
    a_{i,\frac{i}{2}}^2= \frac{i!}{2^i(\frac{i}{2})!(\frac{i}{2})!}\sim \sqrt{\frac{2}{\pi i}}.
    \]
    Thus, we can get the following upper bound
    \begin{align*}
    \sum_{i_1,i_2\in I^*}a_{i_1,\frac{i_1}{2}}^2a_{i_2,\frac{i_2}{2}}^2 
    &\le \left(\sum_{1\le i\le k, i \text{ even}}a_{i,\frac{i}{2}}^2\right)^2-a_{0,0}^2 \\
    &\le \left(\sum_{1\le i\le k, i \text{ even}}\sqrt{\frac{2}{\pi i}}+1\right)^2-1\\
    &\le \frac{2k}{\pi}+2\sqrt{\frac{2k}{\pi}}
    \end{align*}
    and then $|\E[q]|\le C \sqrt{k\Var(q)}$ for some constant $C$.
    \end{proof}

\begin{proof}[Proof of Lemma~\ref{lem:anticonc}]
    Consider the joint distribution $(\lambda,g)\sim N(0,I_2)$. Then $(\lambda,\lambda g) \sim \nu$.
    Let $
    q(\lambda,g) = p(\lambda,\lambda g)$.
    It is equivalent to prove that 
    \[
    \forall \delta>0, ~\sup_{t \in \R} \Pr_{(\lambda,g)\sim \mathcal{N}(0,1)} \Big[\big|q(\lambda,g) -t \big| \le \delta \cdot |\E[q]| \Big] \le C \cdot k\cdot \Big(\sqrt{k}\delta\Big)^{\frac{1}{2k}}.
    \]
By Lemma~\ref{lem:mean-var-ratio}, $|\E[q]|\le C \sqrt{k\Var(q)}$ for some constant $C$.
Then by Carbery-Wright, 
\begin{align*}
    \forall \delta>0, ~
    \sup_{t \in \R} \Pr_{(\lambda,z)\sim \nu} \Big[\big|q(\lambda,z) -t \big| \le \delta \cdot |\E[q]| \Big] 
    &\le
    \sup_{t \in \R} \Pr_{(\lambda,z)\sim \nu} \Big[\big|q(\lambda,z) -t \big|
    \le \delta \cdot c\sqrt{k~\mathrm{Var}(q)} \Big] \\
    &\le C \cdot k\cdot (\sqrt{k}\delta)^{\frac{1}{2k}}.
    \end{align*}
\end{proof}

\subsection{Tukey depth and Existence of Distribution Matching Moments}\label{sec:tukey-depth}
Next we connect the anti-concentration properties from Lemma~\ref{lem:anticonc} to the convex geometry required for our moment-matching proof in Section~\ref{sec:moment-matching}. Our goal is to prove that our target moment vector $\theta$ is in the convex hull of the support of $\Phi(\lambda,z)$.

We first introduce the concept of Tukey depth, a geometric measure of data-centrality. We then use a Carath\'eodory-style theorem (Theorem~\ref{thm:lyons}) that formally links positive Tukey depth to the existence of a vector within the convex hull of a set of random samples.

\begin{defn}[Tukey depth]\label{def:Tukey-depth}
Given a random variable $\Phi$ in $\R^D$ and a point $\theta \in \R^D$, the Tukey-depth of $\theta$ is given by
\begin{equation}\label{eq:Tukey-depth}
h_{\Phi}(\theta) = \inf_{u: \|u\|_2=1} \Pr_{\Phi}\Big[ \iprod{u,\Phi - \theta} \ge 0 \Big]. 
\end{equation}
\end{defn}

The above notion introduced by Tukey~\cite{tukey1975mathematics} plays an important role in statistics. The Tukey depth is also related to Carath\'eodory-style bounds for random polytopes as follows. To prove that $\theta \in \conv(\{\Phi_1, \dots, \Phi_m\})$, it suffices by the hyperplane separator theorem to show that for every direction $u$, there exists $i \in [m]$ such that $\iprod{u,\Phi_i} \ge \iprod{u,\theta}$. Hence, if we sample $m=\tilde{O}(D/h_\Phi(\theta))$ random points $\Phi_1, \dots, \Phi_m$ i.i.d. from the distribution of $\Phi$, one can use Chernoff-Hoeffding bounds along with an $\epsilon$-net over unit vectors in $\R^D$ to argue that $\theta \in \mathrm{conv}\{(\Phi_1, \dots, \Phi_m)\}$ w.h.p. The following theorem in \cite{hayakawa2023a} gives a tighter bound. 

\begin{theorem}[Carath\'eodory bound for random polytopes based on Tukey depth, \cite{hayakawa2023a}]\label{thm:lyons}
Consider a random vector $\Phi$ in $\R^D$ and let $\theta \in \R^D$ have Tukey depth $h_\Phi(\theta)$. Then given $m = (3D+1)/h_\Phi(\theta)$ i.i.d. samples $\Phi_1, \dots, \Phi_m$ from the distribution of $\Phi$, we have 
$$\Pr\Big[\theta \in \mathrm{conv}\Big(\{\Phi_1, \dots, \Phi_m\} \Big)\Big] \ge \frac{1}{2}.$$ 
\end{theorem}

We will apply the above theorem with the random vector $\theta$ being the candidate moment vector of interest, and $\Phi$ being random samples from the distribution $\nu$ to show the existence of a valid distribution that matches the moments given by $\theta$.  

We first show the following lemma that translates the two-sided concentration bound (Lemma~\ref{lem:anticonc}) into one-sided concentration bound, which then gives a lower bound on the Tukey depth
of our target point. %

\begin{lemma}\label{lem:anticonc-onesided}
Suppose $\nu$ is a probability measure over $\R^n$ that satisfies the following anti-concentration statements: for any multivariate polynomial $p:\R^n \to \R$ of degree at most $k$, 
\[
\forall \delta>0, ~\Pr_{x\sim \nu} \Big[\left\lvert p(x) -\E_\nu[p] \right\rvert \le \delta \cdot \sqrt{\Var(p(x))} \Big] \le \eta_k(\delta).
\]
Then for any $\delta_1 >0$ such that $\eta_k(\delta_1)\le 1/2$, we have
\[
\Pr_{x \sim \nu}\Big[  p(x) \ge \E_\nu[p(x)]+\delta \cdot \sqrt{\Var(p(x))} \Big] \ge \frac{\delta_1^2}{16} - \eta_k(\delta).
\]
In particular, there exists $\delta_0>0$ such that
\[
\forall \delta<\delta_0,~~ \Pr_{x \sim \nu}\Big[  p(x) \ge \E_\nu[p(x)] + \delta \cdot \sqrt{\Var(p(x))} \Big] \ge 3 \eta_k(\delta),
\]
by setting $\delta_1=8 \sqrt{\eta_k(\delta)}$. Here  $\delta_0>0$ a constant chosen such that $\eta_k(8\sqrt{\eta_k(\delta_0)})<\frac{1}{2}$.
\end{lemma}
\begin{proof}
Consider the r.v.  $Z = p(x) - \E[p(x)]$. We have $\E[Z]=0$. By the anti-concentration property of $p$ under measure $\nu$, we have 
\begin{align}
\Pr[|Z| \le \delta \sqrt{\Var[Z]}] \le \eta_k(\delta), \text{ and } \Pr[|Z| \le \delta_1 \sqrt{\Var[Z]}] \le \eta_k(\delta_1).\label{eq:intermid-1}
\end{align}

\noindent Furthermore since $\E[Z]=0$, by Cauchy-Schwarz,
\[
\frac{1}{2}\E[|Z|] = \E\Big[ Z ~\mathbf{1}[Z \ge 0] \Big] \le \sqrt{\Var[Z]} \sqrt{\Pr[Z \ge 0]}.
\]
Then
\begin{equation}
\Pr\Big[ Z \ge \E[Z]+\delta \sqrt{\Var[Z]} \Big] \ge \Pr[ Z \ge 0] - \eta_k(\delta) \ge \frac{\E[|Z|]^2}{4 \Var[Z]} - \eta_k(\delta).  \label{eq:intermid-2}
\end{equation}
Moreover 
\begin{align}\E[|Z|]\ge \Pr\left[|Z| \ge \delta_1 \sqrt{\Var[Z]}\right] \cdot \delta_1 \sqrt{\Var[Z]} \ge (1-\eta_k(\delta_1)) \delta_1 \sqrt{\Var[Z]}. \label{eq:intermid-3}
\end{align}
Combining \eqref{eq:intermid-2}, \eqref{eq:intermid-3}, we get
\begin{align*}
\Pr\Big[ Z \ge \E[Z]+\delta \sqrt{\Var[Z]} \Big]& \ge \frac{(1-\eta_k(\delta_1))^2 \delta_1^2}{4} - \eta_k(\delta) \ge \frac{\delta_1^2}{16} - \eta_k(\delta),
\end{align*}
for our choice of parameters. 
\end{proof}

\begin{coro}\label{coro:tukey-depth}
Let $\Phi\in \R^D$ be a random variable with coordinates $\Phi_{\ell_1,\ell_2}=\lambda^{\ell_1}z^{\ell_2}$ where $(\ell_1,\ell_2)\in \mathcal I= \{(i_1,i_2)\neq 0: i_1,i_2\text{ even}, 0\le i_2\le i_1\le k\}$ and $(\lambda,z)\sim\nu$. Then for any $\delta\le 1/k^{Ck^2}$, 
\[
h_\Phi((1+\delta)\E[\Phi]) \ge 2\eta_k(\sqrt{k}\delta) >0.
\]

\end{coro}

\begin{proof}
Let $\beta\in \R^D$ be an arbitrary vector of coefficients with $\norm{\beta}=1$. Then 
    \[
    p(\lambda,z) := \langle \beta, \Phi\rangle = \sum_{\mathcal I}\beta_{\ell_1,\ell_2}\lambda^{\ell_1} z^{\ell_2}.
    \]
    By Lemma~\ref{lem:anticonc-onesided}, we know that 
    \[
\forall \delta<\delta_0,~~ \Pr_{x \sim \nu}\Big[  p(x) \ge \E_\nu[p(x)] +  \delta \cdot \sqrt{\Var(p(x))}\Big] \ge 3 \eta_k(\delta)
\]
where $\delta_0>0$ is chosen such that $\eta_k(8\sqrt{\eta_k(\delta_0)})<\frac{1}{2}$.
Since $(\lambda,g)\sim\mathcal N(0,I_2)$ is log-concave, Carbery-Wright gives $\eta_k(\delta)=C_0 k \delta^{1/k}$ for $q(\lambda,g)$. Because $q(\lambda,g)=p(\lambda,\lambda g)$, the distributional statements for $q$ under $\mathcal N(0,I_2)$ transfer verbatim to $p$ under $\nu$.
Hence 
we can choose $\delta_0=1/k^{Ck^2}$ for some constant $C$ so that it satisfies the condition in Lemma~\ref{lem:anticonc-onesided}.
From the argument in the proof of Lemma~\ref{lem:anticonc}, we have 
$|\E[p]|\le c \sqrt{k\Var(p)}$ for some constant $c$.
So
\begin{align*}
h_\Phi((1+\delta)\E[\Phi]) 
&=\inf_{\beta: \|\beta\|_2=1} \Pr_{\nu}\Big[ \iprod{\beta,\Phi - (1+\delta)\E[\Phi]} \ge 0 \Big]\\
&=\inf_{\beta: \|\beta\|_2=1}\Pr_\nu[p\ge(1+\delta)\E[p]]\\
&\ge\inf_{\beta: \|\beta\|_2=1}\Pr_\nu\left[p\ge \E[p] + c \delta \sqrt{k\Var(p)}\right]\\
&\ge 3\eta_k(c\sqrt{k}\delta)>0.
\end{align*}
\end{proof}
Corollary~\ref{coro:tukey-depth} is the crucial link back to the proof of Theorem~\ref{thm:moment-matching}. It confirms that the specific moment vector $\Phi(\lambda,z)$ with $(\lambda,z)\sim\nu$ from our construction is sufficiently anti-concentrated (via Lemma~\ref{lem:anticonc} and Lemma~\ref{lem:anticonc-onesided}) to guarantee that a point slightly perturbed from its mean has a non-zero Tukey depth. The following lemma shows that this non-zero depth implies the point is in the interior of the convex hull, which provides the ``room'' needed to complete the moment-matching argument in the proof of Theorem~\ref{thm:moment-matching}.

\begin{lemma}\label{lem:depth-interior}
    Let $\Phi$ be a random vector in $\R^D$ such that every affine hyperplane has zero $\Phi$-measure. If $h_\Phi(\theta)>0$, then $\theta$ is in the interior of $\mathrm{conv}(\mathrm{supp}(\Phi))$. In particular, there exists $r=r(h_\Phi(\theta))>0$ such that $B(\theta,r)\subseteq \mathrm{conv}(\mathrm{supp}(\Phi))$.
\end{lemma}

\begin{proof}
    We will prove by contradiction. Suppose $\theta$ is on the boundary of 
    $\mathrm{conv}(\mathrm{supp}(\Phi))$, then a supporting hyperplane $u$ satisfies $\langle u,\phi\rangle\le \langle u,\theta\rangle$ for all $\phi$ in the support, which implies $h_\Phi(\theta)=0$. Therefore, $\theta$ is in the interior of $\mathrm{conv}(\mathrm{supp}(\Phi))$. Since the interior of $\mathrm{conv}(\mathrm{supp}(\Phi))$ is open, there exists $r>0$ such that $B(\theta,r)\subseteq \mathrm{conv}(\mathrm{supp}(\Phi))$.
\end{proof}

\section{Extending Hardness to Higher Degrees via Bounded LDA}\label{sec:higher-degrees}
In this section, we prove the first part of Theorem~\ref{ithm:main2}. We extend our hardness result to polynomials of degree $k=\poly(n)$ demonstrating that the LDA remains bounded in this regime. This effectively rules out any efficient algorithm that relies solely on low-degree polynomials.
We will use the same null distribution $Q$ as in Section~\ref{sec:moment-matching}, i.e., $Q$ is a scale mixture of Gaussians with Gaussian scales $\mu=\mathcal{N}(0,1)$. Define the planted distribution $P = (1-\alpha)Q + \alpha 0$. %

\begin{theorem}[LDA Bound up to degree $k= n^{\Omega(1)}$]\label{thm:LDA}
For polynomials of degree at most $k$ on $m$ samples,
\[
\LDA^{(m)}_{\le k}(P,Q) \le \sqrt{(1+C^2\alpha^2 k)^m-1}.
\]
In particular, if $k = O(1/(\alpha^2 m))$, then $\LDA^{(m)}_{\le k}(P,Q) = O(1)$. 
   
\end{theorem}

The proof of the theorem has two steps. We first show a single-sample LDA bound and then lift it to $m$ samples.
The following lemma shows how low-degree advantage scales from one sample to $m$ i.i.d. samples for any distribution $P$ and $Q$.
\begin{restatable}{lemma}{liftLDA}
\label{lem:1-to-m-sample}
    Let $P,Q$ be distributions on $\R^n$ with finite moments up to degree $2k$. Suppose the single-sample low-degree advantage is $\LDA_{\le k}^{(1)}(P,Q) = \delta$. Then for any $m$,
    \[
    \LDA_{\le k}^{(m)}(P,Q)^2 \le (1+\delta^2)^m-1.
    \]
\end{restatable}
\begin{proof}
Let $\{\psi_j:j\in \mathcal{J}\}$ be an orthonormal basis of polynomials of degree at most $k$ under $Q$ with $\psi_0 = 1$. 
For a multi-index $\gamma=(\gamma_1,\dots,\gamma_m)$, define
\[
\Psi_\gamma(X_1,\dots,X_m) = \prod_{i=1}^m\psi_{\gamma_i}(X_i).
\]
Because $Q^{\otimes m}$ is a product measure, $\Psi_\gamma$ are orthonormal with respect to $Q^{\otimes m}$. For any polynomial $f$ of degree at most $k$, subtracting the constant $\E_{Q^{\otimes m}}[f]$ does not change LDA so we may assume $\E_{Q^{\otimes m}}[f]=0$.
Then $f$ can be written as 
\[
f = \sum_{\gamma\neq 0} c_\gamma \Psi_\gamma
\]
where $\gamma\neq 0$ means not all coordinates of $\gamma$ are zero. Let $s(\gamma)$ be the number of nonzero entries in $\gamma$. Then $s(\gamma)\le k$ since the degree of $f$ is at most $k$. By orthonormality,
\[
\Var_{Q^{\otimes m}}(f)=\E_{Q^{\otimes m}}[f^2]=\sum_{\gamma\neq 0}c_\gamma^2.
\]
Let $\mu_j=\E_P[\psi_j]$. We have
\[
\E_{P^{\otimes m}}[\Psi_\gamma] = \prod_{i=1}^m \mu_{\gamma_i}.
\]
Since $\E_{Q}[\psi_j]=0$ for $j\neq0$,
\[
\E_{P^{\otimes m}}[f] -\E_{Q^{\otimes m}}[f]= \sum_{\gamma\neq 0} c_\gamma \prod_{i=1}^m \mu_{\gamma_i}.
\]
By Cauchy-Schwarz,
\[
\LDA_{\le k}^{(m)}(P,Q)^2 =\frac{(\E_{P^{\otimes m}}[f] -\E_{Q^{\otimes m}}[f])^2}{\Var_{Q^{\otimes m}}(f)}
=\frac{\left(\sum_{\gamma\neq 0} c_\gamma \prod_{i=1}^m \mu_{\gamma_i}\right)^2}{\sum_{\gamma\neq 0}c_\gamma^2}
\le\sum_{\gamma\neq 0}\prod_{i=1}^m \mu_{\gamma_i}^2.
\]
Similarly, when $m=1$, we have 
\[
\LDA_{\le k}^{(1)}(P,Q)^2 = \frac{\left(\sum_{j} c_j\mu_{j}\right)^2}{\sum_{j}c_j^2} \le \sum_j \mu_j^2.
\]
When $c$ and $\mu$ are linearly dependent, the equation holds. Thus $\LDA_{\le k}^{(1)}(P,Q)^2 =\sum_j \mu_j^2=\delta^2$.
Fix a subset $S\subseteq[m]$ of size $s$. Group $\gamma$ by the nonzero coordinates by the set $S$,
\[
\sum_{\substack{\gamma:\{i:\gamma_i\neq 0\}=S, \\ \sum_{i\in S}\mathrm{deg}(\psi_{\gamma_i})\le k}}\prod_{i\in S} \mu_{\gamma_i}^2
\le \sum_{j\in\mathcal{J}}\prod_{i\in S} \mu_{j}^2 
 =\prod_{i\in S}\left(\sum_{j\in\mathcal{J}} \mu_j^2\right)=\left(\sum_{j\in\mathcal{J}} \mu_j^2\right)^s=\delta^{2s}.
\]
Sum over all possible $S$ with size $s\le k$,
\[
\LDA_{\le k}^{(m)}(P,Q)^2 \le \sum_{s=1}^k\binom{m}{s}\delta^{2s}\le (1+\delta^2)^m-1.
\]
\end{proof}

We now prove Theorem~\ref{thm:LDA}. 

\begin{proof}[Proof of Theorem~\ref{thm:LDA}]
We first bound the single-sample low-degree advantage. Let $f$ be any polynomial of degree at most $k$ on $\R^n$. Since adding a constant to $f$ does not change LDA, we may assume the constant term of $f$ is zero. Then
\[
\E_P[f]-\E_Q[f] = -\alpha\E_Q[f].
\]
Let $X\sim Q$. We can write $X=\lambda g$ with independent r.v. $g\sim \mathcal{N}(0,I_n)$ and $\lambda\sim \mathcal{N}(0,1)$. Define $h(\lambda)=\E_g[f(\lambda g)]$. Then $\E_\lambda[h(\lambda)]=\E_Q[f(X)]$ and by the law of total variance,
\[
\Var_Q(f(X))\ge\Var_\lambda(h(\lambda)).
\]
Moreover $h(0)=0$ i.e., $h(\lambda)$ has no constant term.  Hence we have reduced it to an analogous claim about the one dimensional scale distribution $\lambda$. We will now upper bound $|\E_\lambda[h(\lambda)]|/ \sqrt{\Var_\lambda[h(\lambda)]}$. 
Apply Lemma~\ref{lem:mean-var-ratio} to $h$ by taking $q(\lambda,g) = h(\lambda)$,
\[
|\E_Q[f(X)]|=\E_\lambda[h(\lambda)]\le C\sqrt{k\Var_\lambda(h(\lambda))}\le C\sqrt{k\Var_Q(f(X))}.
\]
Thus, 
\[
\LDA^{(1)}_{\le k}(P,Q) \le C\alpha\sqrt{k}.
\]
Apply Lemma~\ref{lem:1-to-m-sample} with $\delta = C\alpha\sqrt{k}$,
\[
\LDA^{(m)}_{\le k}(P,Q) \le \sqrt{(1+C^2\alpha^2 k)^m-1}.
\]
\end{proof}

In the above proof, we crucially used the property that $X \sim Q$ can be expressed 
as the product $X = \lambda g$ of two independent random variables, where 
$\lambda \in \mathbb{R}$ represents a scale distribution drawn from $\mu$, 
and $g \in \mathbb{R}^n$ is drawn from any rotationally invariant distribution. 
The specific choices of $g$ and $\lambda$ being Gaussian were for convenience. 
We now explain precisely which structural property of the scale distribution 
$\mu$ governs the LDA bound, and how the argument generalizes.

The central step in the proof reduces the $n$-dimensional single-sample 
advantage to a bound on the ratio $|E_\mu[h(\lambda)]| / \sqrt{\mathrm{Var}_\mu(h(\lambda))}$
for univariate degree-$k$ polynomials $h(\lambda)$ satisfying $h(0) = 0$. 
The ratio $|E_\mu[h]|/\sqrt{\mathrm{Var}_\mu(h)}$ has a clean
characterization via the orthonormal polynomials $\{p_j\}_{j \geq 0}$ for 
the measure $\mu$. 
Expanding $h(\lambda) = \sum_{j=0}^k c_j p_j(\lambda)$, orthonormality gives 
$E_\mu[h] = c_0$ and $\mathrm{Var}_\mu(h) = \sum_{j=1}^k c_j^2$. 
The constraint $h(0) = 0$ yields
\[
    c_0 = -\sum_{j=1}^k c_j\, p_j(0).
\]
By Cauchy--Schwarz,
\[
    c_0^2 \;\leq\; \Bigl(\sum_{j=1}^k c_j^2\Bigr) \Bigl(\sum_{j=1}^k p_j(0)^2\Bigr) 
    \;=\; \mathrm{Var}_\mu(h) \cdot \sum_{j=1}^k p_j(0)^2,
\]
so that
\begin{equation}\label{eq:christoffel-bound}
    \frac{|E_\mu[h(\lambda)]|}{\sqrt{\mathrm{Var}_\mu(h(\lambda))}} 
    \;\leq\; \sqrt{\sum_{j=1}^k p_j(0)^2},
\end{equation}
with equality achieved when $c_j \propto p_j(0)$ for $j \geq 1$.
The quantity
\[
    \mathcal{K}_k(\mu, 0) \;:=\; \sum_{j=0}^k p_j(0)^2
\]
is the Christoffel sum for $\mu$ at $0$ (the inverse of the 
Christoffel function $\lambda_k(\mu,0)^{-1}$). 
Since $p_0 \equiv 1$, the bound~\eqref{eq:christoffel-bound} becomes
$|E_\mu[h]|/\sqrt{\mathrm{Var}_\mu(h)} \leq \sqrt{\mathcal{K}_k(\mu,0) - 1}$.
The growth rate of $\mathcal{K}_k(\mu,0)$ with $k$ is therefore the 
fundamental quantity controlling the LDA bound.
In particular, one can choose any scale distribution for which this quantity grows slowly enough with $k$ yields the same kind of bound on the single-sample low-degree advantage, and hence the same lifting argument to $m$ samples.

\section{Simple Noise-Tolerant Algorithm} \label{sec:algo}

In this section we show that there are simple efficient algorithms that are not based on the low-degree polynomial method that can solve this problem efficiently. Consider the setting when the $\alpha=1/\poly(n)$ fraction lies on a $d \le 1$ dimensional subspace. On the one hand, the low-degree method fails even when an $\alpha$ fraction of points lie on the subspace. On the other hand, our simple algorithm samples points and checks whether there are $(d+1)$ points ({\em one} point if $d=0$) that are close to a $d$-dimensional subspace (close to $0$ if $d=0$). Moreover, it can achieve tolerance to worst-case errors, and rerandomization.  

An elegant and clever algorithm for robust subspace recovery due to Hardt and Moitra~\citep{HardtM2013} uses anti-concentration properties of $(1-\alpha)$ points not on the subspace $S$: it gives efficient algorithms when $\alpha \ge \frac{\dim(S)}{n}$. Subsequent work by Bhaskara, Chen, Perreault and Vijayaraghavan~\citep{BCPV} gave an improved algorithm in the smoothed analysis setting where the points not on the subspace are randomly perturbed (by a random vector of average length $\rho = 1/\poly(n)$), and gave guarantees of recovery when $\alpha \gtrsim (d/n)^\ell$ for any constant $\ell>0$ in $(nd)^{O(\ell)}$ time. In our setting where $d=O(1)$ and $\alpha=1/n^{O(1)}$, we can pick an appropriate $\ell=O(1)$ to get algorithmic guarantees. However this algorithm is only moderately noise tolerant: every point can be adversarially perturbed (moved) by $1/n^{O(\ell)}$. See also \citep{bakshi2021list, gao2026ellipsoid} for algorithmic results in other parameter regimes. We will instead give a simple algorithm specialized for the $d=O(1)$ setting that is elementary and also allows points to be moved by an amount $\eps = \Omega(1)$.

\subsection{Robust Subspace Recovery with $\eps$ Relative Error Perturbations. }

We define the hypothesis testing version of the problem below. We are given samples from a distribution $\calD$ over $\R^n$, and the goal is to distinguish between the following two hypothesis:
\begin{itemize}
\item {\bf \NO:} $\calD = \Qrot$ where $\Qrot$ is the rotationally invariant distribution defined in Section~\ref{sec:moment-matching}.
\item {\bf \YES:} $\calD$ is {\em any} distribution with at least an $\alpha$ probability mass supported on a planted subspace $S$ of dimension $d$. 
\end{itemize}
Note that the \YES~ case above is more general that the planted distribution in the lower bound example in Section~\ref{sec:moment-matching}.  Moreover, our algorithm will work even when the samples drawn from $\calD$ are then adversarially perturbed as follows. 
\begin{defn}[Adversarial $(\eps,p)$-noisy sample from $\calD$] \label{def:noisysamples}
Draw $m$ samples $x_1, \dots, x_m \in \R^n$ drawn i.i.d. from $\calD$. Then {\em an adversary} can perform the following for $p \in [0,1)$:
\begin{enumerate}
    \item {\em Re-randomize to draw from $\Qrot$: } For each $i \in [m]$, with probability $p$ we redraw $x_i \sim \Qrot$, and with probability $(1-p)$ it remains unaltered.  
    \item {\em Adversarial $\eps$-perturbation to each point:} for each $i \in [m]$ the sample $x_i$ is adversarially perturbed  to form $\tilde{x}_i = x_i + z_i$ where $\norm{z_i}_2 \le \eps \norm{x_i}_2$.
\end{enumerate} 
The input are the corrupted samples $\tilde{x}_1, \dots, \tilde{x}_m$.
\end{defn}
Note that the adversary observes the $m$ samples $x_1, \dots, x_m$, and the corruptions to the samples can be correlated arbitrarily. We remark that the low-degree polynomial method continues to fail in this model.  %
We analyze the following simple algorithm. 

\begin{figure}[ht]
\begin{tcolorbox}
\begin{center}
    \textbf{Noise Tolerant Algorithm} 
\end{center}
\textbf{Input:} Given $m$ samples $\tilde{x}_1, \dots, \tilde{x}_m$, dimension $d$, lower bound $\alpha_0 \in (0,1)$ for parameter $\alpha$, upper bound $p_0 \in [0,1)$ on parameter $p$ \\
\textbf{Output:} \NO, or \YES (along with a subspace $S$ of dimension $d$).
\begin{enumerate}
\item Normalize the vectors $a_i = \tilde{x}_i / \norm{\tilde{x}_i}$ for all $i \in [m]$. 
\item Run over all $\binom{m}{d+1}$ choices $T \subset [m]$ of size $|T|=d+1$ 
\item \begin{enumerate}
    \item[(i)] Let $A_T$ be the matrix formed by the corresponding points $A_T=(a_i: i \in T) \in \R^{n \times (d+1)}$.
    \item [(ii)] If $\sigma_{d+1}(A_T) \le \frac{1}{2} $, output $\YES$ and return. \\
    \% (If the subspace is desired, output the top-$d$ left singular space of $A_T$).  
\end{enumerate} 
\item Output $\NO$.
\end{enumerate}
\end{tcolorbox}
\caption{Noise Tolerant Algorithm for detecting if an $\alpha$ fraction of the points lie close to a $d=O(1)$-dimensional subspace}
\label{fig:noisealgorithm}
\end{figure}

\begin{restatable}[Noise tolerant algorithm with relative error perturbations]{theorem}{algo}
\label{thm:algo}
There exists a universal constants $c, C>0$ such that the following holds when $m \ge C (d+1)/ (\alpha (1-p))$ and any $0\le \eps\le \frac{1}{8(d+1)}$. There is an algorithm (given in Figure~\ref{fig:noisealgorithm}) that runs in time $O(\binom{m}{d+1})$ that given as input noisy samples $\tilde{x}_1, \dots, \tilde{x}_m$ as described above with $\norm{\tilde{x}_i - x_i } \le \eps \norm{x_i}$, can solve with high probability the testing problem. 
\end{restatable}
We remark that the algorithm allows $\eps$ to be as large as a constant when the dimension $d=O(1)$ even when the perturbations are worst-case. 

The argument for the \NULL~ setting (soundness) follows from the following lemma. 
\begin{lemma}[\NULL~ setting]\label{lem:alg:soundness}
In the setting of Theorem~\ref{thm:algo}, there exists a constant $c'>0$ such that with probability at least $1-m^2 \exp(-c'n)$ we have
$$\forall T \subset [m] \text{ with } |T|=d+1, ~~\sigma_{d+1}(A_T) > \frac{1}{2}.$$
\end{lemma}
\begin{proof}
The proof follows from the {\em incoherence} of the matrix $A$. 

\begin{claim}\label{claim:noise:incoherence}
There exists universal constants $c,c'>0$ such that the set of vectors $x_1, \dots, x_m \sim \Qrot$, the (corrupted) normalized vectors $\{a_i = \frac{(x_i+z_i)}{\norm{x_i+z_i}}: i \in [m]\}$ satisfies with probability at least $1-O\big(m^2  \exp(-c' n)\big)$ that 
\begin{equation}\label{eq:claim:condition}
\forall i \ne  j \in [m], ~ |\iprod{a_i, a_j}| \le \frac{1}{(1-\eps)^2}\Big(\frac{c\sqrt{\log m}}{\sqrt{n}} + 2\eps+\eps^2\Big). 
\end{equation}
\end{claim}
\noindent {\em Proof of the Claim.} From standard concentration properties of random unit vectors we have for some universal constant $c>0$, 
\begin{align}
\Pr_{\hat{x}, \hat{y} \sim_{unif} \mathbb{S}^{n-1}} \Big[ |\iprod{\hat{x}, \hat{y}} | \ge   \frac{c\sqrt{\log m}}{\sqrt{n}}\Big] &\le \frac{1}{m^3}, ~~\text{ and hence  by a union bound},\nonumber\\ 
\forall i \ne j \in [m], ~\Pr_{x_i, x_j \sim \Qrot} \Big[ \Big|\Big\langle \frac{x_i}{\norm{x_i}}, \frac{x_j}{\norm{x_j}} \Big\rangle \Big| \le   \frac{c\sqrt{\log m}}{\sqrt{n}}\Big] &\ge 1- \frac{1}{m}. \label{eq:soundness:1}
\end{align}
Let us condition on the event in \eqref{eq:soundness:1}. Also note that $\norm{x_i+z_i} \ge (1-\eps) \norm{x_i}$ Then we have for all $i \ne j \in [m]$ 
\begin{align*}
|\iprod{a_i, a_j}| &= \Big|\Big\langle\frac{x_i + z_i}{\norm{x_i+z_i}},\frac{x_j+z_j}{\norm{x_j+z_j}}\Big\rangle\Big| \le \frac{|\iprod{x_i, x_j}|}{(1-\eps)^2\norm{x_i} \norm{x_j}} + \frac{|\iprod{z_i, x_j}|+|\iprod{x_i, z_j}|+|\iprod{z_i,z_j}|}{(1-\eps)^2\norm{x_i}\norm{x_j}} &\\
&\le  \frac{c\sqrt{\log m}}{(1-\eps)^2\sqrt{n}} + \frac{\norm{z_i}\norm{x_j}+\norm{x_i}\norm{z_j}+\norm{z_i}\norm{z_j}}{\norm{x_i} \norm{x_j}} \le  \frac{c\sqrt{\log m}}{(1-\eps)^2\sqrt{n}} + \frac{2\eps+\eps^2}{(1-\eps)^2}&,
\end{align*}
where the last line used \eqref{eq:soundness:1} and $\norm{z_i}\le \eps \norm{x_i}$. 

Now to complete the rest of the proof, consider any $T \subset [m]$ with $|T|=d+1$, and a test unit vector $\beta \in \R^{d+1}$ for the matrix $A_T$. We have 
\begin{align*}
\norm{A_T \beta}_2^2 & = \Big\| \sum_{i \in T} \beta_i a_i \Big\|_2^2 \ge \sum_{i \in T} \beta_i^2 \norm{a_i}_2^2 - \sum_{i \ne j}  |\beta_i| |\beta_j| |\iprod{a_i, a_j}|\\
&\ge \sum_{i \in T} \beta_i^2  - \sum_{i \ne j \in T}   |\beta_i| |\beta_j|  \Big(\frac{c \sqrt{\log m}}{\sqrt{n}(1-\eps)^2} + \frac{2\eps+\eps^2}{(1-\eps)^2}\Big) \\
&\ge 1-  \norm{\beta}_1^2  \Big(\frac{c \sqrt{\log m}}{\sqrt{n}(1-\eps)^2} + \frac{2\eps+\eps^2}{(1-\eps)^2}\Big) \ge 1 - \frac{(d+1)}{(1-\eps)^2}\Big(\frac{c \sqrt{\log m}}{\sqrt{n}} + 2\eps+\eps^2\Big)\\
&>\frac{1}{2} \qquad \text{ since } \eps < \frac{1}{8(d+1)}, \text{ and } m < 2^{o(n)}.
\end{align*}

\end{proof}

\begin{proof}
To prove the algorithm works, we will show that there is (a) an approximate linear dependence between $d+1$ of the points sampled from the $\alpha$ portion in the subspace in the \YES~ case, and (b) no approximate linear dependence between any set of $d+1$ points in the \NO~ case. 

\paragraph{\YES~ case (completeness): } 
Let $x_i$ be the uncorrupted sample from $\calD$ corresponding to $\tilde{x}_i$, and let $\tilde{x}_i = x_i + z_i$ with $\norm{z_i} \le \eps \norm{x_i}$. 
Given $m \ge C (d+1)/(\alpha (1-p))$, there exists in expectation at least $Cd$ points from $\calD$ lying in a subspace $S$; hence with high probability there exists $T^\star \subset [m]$ with $|T|=d+1$ corresponding to points $\{x_i : i \in T \}$ that belong to a subspace $S$ of dimension $d$.  Recall that $a_i = \tilde{x}_i / \norm{\tilde{x_i}}$. 

The submatrix $M_{T^\star} = (\frac{x_i}{\norm{x_i}} : i \in T^\star)$ has $\sigma_{d+1} (M_{T^{\star}})= 0$. Let $\beta \in \R^{d+1}$ be a unit vector satisfying $M_{T^{\star}} \beta =0$. Then
\begin{align*}
A_{T^{\star}} \beta &= \sum_{i \in T^{\star}} \beta_i a_i = \sum_{i \in T^{\star}} \beta_i \frac{x_i+z_i}{\norm{x_i+z_i}} = \sum_{i \in T^{\star}} \beta_i \frac{x_i}{\norm{x_i+z_i}} + \beta_i \frac{z_i}{\norm{x_i+z_i}}&\\
&=  \sum_{i \in T^{\star}} \beta_i \Big( \frac{x_i}{\norm{x_i+z_i}} - \frac{x_i}{\norm{x_i}}\Big) + \sum_{i \in T^\star} \beta_i \frac{z_i}{\norm{x_i+z_i}} & \text{ (since } M_{T^\star} \beta=0 \text{)}\\
\norm{A_{T^{\star}} \beta}_2 &\le \sum_{i \in T^\star} |\beta_i| \Big(\frac{\norm{x_i}}{\norm{x_i+z_i}} -1\Big) + \eps \sum_{i \in T^\star} |\beta_i|\frac{\norm{x_i}}{\norm{x_i+z_i}}& \text{ (since } \norm{z_i} \le \eps \norm{x_i}\text{)}\\
&\le \Big(\sum_{i \in T^{\star}} |\beta_i|\Big) \Big( \frac{2\eps}{1-\eps}  \Big) \le \frac{2\eps \sqrt{d+1}}{1-\eps} < \frac{1}{2}& \text{ (since } \eps<\frac{1}{8(d+1)} \text{)}.
\end{align*}

\paragraph{\NO~ case (soundness): } This follows from Lemma~\ref{lem:alg:soundness}. This completes the proof.

\end{proof}

\subsection{Model with Additive Perturbations}\label{sec:additive}

In this section, we gave a variant of the algorithm that also works with worst-case additive perturbations of bounded magnitude. 

We consider the same hypothesis testing problem as before: We are given samples from a distribution $\calD$ over $\R^n$. The goal is to distinguish between the following two hypothesis:
\begin{itemize}
\item {\bf \NO:} $\calD = \Qrot$ where $\Qrot$ is the rotationally invariant distribution defined in Section~\ref{sec:moment-matching}.
\item {\bf \YES:} $\calD$ is {\em any} distribution with at least an $\alpha$ probability mass supported on a planted subspace $S$ of dimension $d$. 
\end{itemize}

We now consider the setting where the samples drawn from $\calD$ are then adversarially perturbed as follows. 
\begin{defn}[Adversarial additive $(\eta,p)$-noisy sample from $\calD$] \label{def:addnoisysamples}
Draw $m$ samples $x_1, \dots, x_m \in \R^n$ drawn i.i.d. from $\calD$. Then {\em an adversary} can perform the following for $\eta,p \in [0,1)$:
\begin{enumerate}
    \item {\em Re-randomize to draw from $\Qrot$: } For each $i \in [m]$, with probability $p$ we redraw $x_i \sim \Qrot$, and with probability $(1-p)$ it remains unaltered.  
    \item {\em Adversarial $\eta$-perturbation to each point:} for each $i \in [m]$ the sample $x_i$ is adversarially perturbed  to form $\tilde{x}_i = x_i + z_i$ where $\norm{z_i}_2 \le \eta$.
\end{enumerate} 
The input are the corrupted samples $\tilde{x}_1, \dots, \tilde{x}_m$.
\end{defn}

We analyze the following simple algorithm. 

\newcommand{\Count}{\textrm{Count}}
\begin{figure}[ht]
\begin{tcolorbox}
\begin{center}
    \textbf{Noise Tolerant Algorithm for Additive Perturbations} 
\end{center}
\textbf{Input:} Given $m$ samples $\tilde{x}_1, \dots, \tilde{x}_m$, dimension $d$,  parameter $\alpha \in (0,1)$, parameter $p \in (0,1)$, and threshold $\tau>0$ \\
\textbf{Output:} \NO, or \YES (along with a subspace $S$ of dimension $d$).
\begin{enumerate}
\item Subsample $J \subseteq [m]$ of size $|J| = 2\log (4/\delta) (d+1) /(\alpha (1-p))$. 
\item Run over all $\binom{m}{d+1}$ choices $T \subset J$ of size $|T|=d+1$ 
\item \begin{enumerate}
    \item[(i)] Let $A_T$ be the matrix formed by the corresponding points $A_T=(\tilde{x}_i: i \in T) \in \R^{n \times (d+1)}$.
    \item [(ii)] If $\sigma_{d+1}(A_T) \le \tau $, output \YES. \\
    \% (If the subspace is desired, output the top-$d$ left singular space of $A_T$).  
\end{enumerate} 
\item Output \NO %
\end{enumerate}
\end{tcolorbox}
\caption{Algorithm for detecting even under {\em additive perturbations} if an $\alpha$ fraction of the points lie close to a $d=O(1)$-dimensional subspace}
\label{fig:addnoisealgorithm}
\end{figure}

\begin{restatable}[Noise tolerant algorithm with additive error perturbations]{theorem}{algotwo}
\label{thm:algo:additive}
There exists a universal constants $c,c'>0, C\ge 1$ such that the following holds when $m \ge  2(d+1) \log(4/\delta)/ (\alpha (1-p))$ and any $\delta \in (0,1), 0\le \eta\le \frac{\alpha (1-p)\sqrt{n}}{C (d+1) \log(4/\delta)^2}$. Then there is an algorithm (given in Figure~\ref{fig:addnoisealgorithm} with $\tau = \frac{\alpha (1-p) \sqrt{n}}{c\log^2(4/\delta)}$) running in time $O(\binom{m}{d+1})$ that given as input $\tilde{x}_1, \dots, \tilde{x}_m$ generated as described above with $\norm{\widetilde{x}_i - x_i}_2 \le \eta$ can solve the testing problem with probability at least $1-\delta$.
\end{restatable}

In this model, the average length of a point drawn from $\Qrot$ is $\E[\norm{X}_2^2]=\Theta(n)$. Moreover, an $\alpha$ fraction of the points drawn from $\Qrot$ have length $\|X\|_2 \le  c_\star \alpha \sqrt{n}$ for some absolute constant $c_\star>0$-- hence when $\eta \ge c_\star \alpha \sqrt{n}$, the hypothesis testing problem becomes statistically impossible. Theorem~\ref{thm:algo:additive} shows that the algorithm handles $\eta$ that is within a constant factor (for $d=O(1)$) of the best possible bound of $c_\star \alpha \sqrt{n}$.   

The following claim shows that for most sample points drawn from $\Qrot$, their length is at least $2 \tau$. 
\begin{claim}\label{claim:lengthlb}
There exists a universal constant $c>0$ such that for any $\gamma \in (0,1/2)$, we have
\begin{equation}\label{eq:Qrot:lengthlb}
\Pr_{X \sim \Qrot}\Big[ \norm{X}_2 \le \gamma \sqrt{n} \Big] \le \min\{c \gamma,1\}. 
\end{equation}
In particular, by setting $\gamma =  \alpha(1-p)/(4c\log^2(4/\delta) (d+1))$, we get that the fraction of points drawn from $\Qrot$ whose length $\norm{X}_2 \le \frac{\alpha (1-p) \sqrt{n}}{2c \log(4/\delta)^2}$ is at most $\frac{\alpha (1-p)}{4 \log(4/\delta)^2}$. 
\end{claim}
\begin{proof}
This follows directly from the small ball probability bound for a standard Gaussian i.e., $\Pr_{g \sim N(0,1)}[|g| \le \delta/2  ]\le c \delta/2$, along with length concentration of a standard spherical Gaussian.  
\end{proof}

We now proceed to the proof of Theorem~\ref{thm:algo:additive}.
\begin{proof}[Proof of Theorem~\ref{thm:algo:additive}]
To prove the algorithm works, we will show that there is (a) an approximate linear dependence between many $(d+1)$-tuple of the points sampled from the $\alpha$ portion in the subspace in the \YES~ case, and (b) approximate linear dependence between much fewer set of $(d+1)$-tuples of the samples in the \NO~ case. 

\paragraph{\YES~ case (completeness): } 
Let $x_i$ be the uncorrupted sample from $\calD$ corresponding to $\tilde{x}_i$, and let $\tilde{x}_i = x_i + z_i$ with $\norm{z_i} \le \eta$. 
Given $m \ge   2\log(4/\delta) (d+1)/(\alpha (1-p))$, there exists in expectation $(1-p)\alpha |J| \ge 2(d+1) \log(4/\delta)$ points from $\calD$ lying in a subspace $S$. Hence with probability $1-\delta/4$ there are at least $d+1$ samples $J_S \subseteq J$ lying in subspace $S$. %
For any tuple $T \subseteq J_S$ with $|T|=d+1$, we have that $\sigma_{d+1}((x_i : i \in T))=0$. By Weyl's inequality, we have $\sigma_{d+1}(A_T) \le \eta \sqrt{d+1} \le \tau/2$, as required. 

\paragraph{\NO~ case (soundness): } The soundness analysis is very similar to the soundness analysis of Theorem~\ref{thm:algo}. We know from Claim~\ref{claim:lengthlb}, that for at least $1-\alpha/4$ fraction of the samples $\{x_i: i \in J\}$ drawn from $\Qrot$, we have $\norm{x_i}_2 \ge 2 \tau$. %
For these samples which constitute at least $(1-\alpha/4)$ fraction, the additive perturbations can be viewed as multiplicative perturbations, and we can apply Claim~\ref{claim:noise:incoherence} to get a bound. 

\newcommand{\Jbig}{J_{\textrm{big}}}

Formally, among the $m$ samples, let $\Jbig = \{i \in J: \norm{x_i}_2 \ge 2 \tau\}$.  Consider a modified set of points $x'_1, \dots, x'_m$ where for each $i \in J$ we have $x'_i = \tilde{x}_i$ if $i \in \Jbig$ and $x'_i = x_i$ if $i \notin \Jbig$. First we observe that these points $\{x'_i: i \in J\}$ satisfy the conditions of the Lemma~\ref{lem:alg:soundness} as the perturbations satisfy $\norm{x'_i - x_i}_2 < \frac{1}{8(d+1)} \norm{x_i}_2$. Let $A'$ be the matrix with its $i$th column of $A'$ equal to $x'_i/ \norm{x'_i}$. Hence we can now apply Lemma~\ref{lem:alg:soundness} to conclude that for some absolute constant $c'>0$, we have with probability at least $1-O(m^2 \exp(-c'n))$ that 
$$ \forall T \subseteq [J] \text{ s.t. } |T|=d+1, ~  \sigma_{d+1}(A'_T) >1/2.$$

Moreover, we have from Claim~\ref{claim:lengthlb} that 
$$\E[|J \setminus \Jbig|] \le \frac{\alpha (1-p)}{4 (d+1)\log(4/\delta)^2} \cdot \frac{2 \log(4/\delta) (d+1)}{\alpha (1-p)} \le \frac{1}{2 \log(4/\delta)}. $$
Hence with probability at least $1-\delta/2$, we have that $\Jbig=J$. In other words, with probability at least $1-\delta/2$, $x'_i = \tilde{x}_i~ \forall i \in J$. Hence we have with probability at least $1-\delta/2 - m^2 \exp(-c' n)$, 
$$ \forall T \subseteq J \text{ s.t. } |T|=d+1, ~~\sigma_{d+1}(A_T) \ge \min_{i \in T} \norm{\tilde{x}_i}_2 \cdot \sigma_{d+1}(A'_T) \ge \frac{1}{2} \cdot 2\tau \ge \tau. $$
This  concludes the proof. 
\end{proof}

\section*{Acknowledgements}
We thank Alex Wein, Miklos Racz and Sidhanth Mohanty for feedback on an early draft, and Sanchit Kalhan for helpful discussions related to the moment matching bounds. 
The authors were supported by the NSF-funded Institute for Data, Econometrics, Algorithms and Learning (IDEAL) through the grant NSF ECCS-2216970. 

\printbibliography

\end{document}